\documentclass[lettersize,journal]{IEEEtran}

\newcommand{\linelabel}[1]{}
\usepackage{graphicx}
\usepackage{breqn}
\usepackage{xcolor}
\usepackage{paralist}
\usepackage{amssymb}
\usepackage{amsthm}
\usepackage{amsmath}
\usepackage{subcaption}
\usepackage{multirow}
\usepackage{mathtools}
\usepackage{hyperref}
\usepackage{tabularx}
\usepackage{array}
\definecolor{costred}{HTML}{B85450}
\definecolor{costblue}{HTML}{6C8EBF}

\usepackage[ruled,vlined, linesnumbered]{algorithm2e}

\newtheorem{theorem}{Theorem}
\newtheorem{assumption}{Assumption}
\newtheorem{corollary}{Corollary}
\newtheorem{definition}{Definition}
\newtheorem{proposition}{Proposition}
\newtheorem{lemma}{Lemma}
\newtheorem{remark}{Remark}
\newtheorem{example}{Example}
\usepackage{xcolor}
\usepackage{booktabs}
\usepackage{xspace}

\newcommand{\tk}[1]{\textcolor{black}{#1}}

\newcommand{\pbupdate}[1]{\textcolor{black}{#1}}

\newcommand{\rtk}[1]{{\color{black}#1}}
\newcommand{\rgp}[1]{\textcolor{black}{#1}} 

\usepackage{newunicodechar} 
\newunicodechar{□}{\Box} 
\newunicodechar{◇}{\Diamond} 
\newunicodechar{○}{\circ}

\newcommand{\fw}{\textsc{DiffTilt}\xspace}
\newcommand{\alg}{\fw\xspace}

\begin{document}

\title{
Diffusion-Guided Search via Exponential Tilting (\fw): An Application to Falsification of Safety-Critical Systems
}
\author{
Tanmay~Khandait,
Preetom~Biswas,
Hideki~Okamoto,
Bardh~Hoxha,
Georgios~Fainekos,
and~Giulia~Pedrielli%
\thanks{T. Khandait, P. Biswas, and G. Pedrielli are with the School of Computing and Augmented Intelligence, Arizona State University, Tempe, AZ, USA (e-mail: tkhandai@asu.edu, pbiswa11@asu.edu, giulia.pedrielli@asu.edu).}%
\thanks{H. Okamoto, B. Hoxha, and G. Fainekos are with Toyota Motor North America, Research \& Development, Ann Arbor, MI 48105, USA (e-mail: hideki.okamoto@toyota.com, bardh.hoxha@toyota.com, georgios.fainekos@toyota.com).}%
}

\markboth{}%
{Diffusion-Guided Search via Exponential Tilting}


\maketitle

\begin{abstract}

\rtk{Discovering \linelabel{ln:abs-1-start} rare safety-critical failures in autonomous and cyber-physical systems is a fundamental challenge in verification and validation. Existing falsification approaches rely on conditional sampling strategies that factor the joint distribution over environments and system executions, and therefore suffer from multiplicative rarity effects: the simultaneous scarcity of failure-inducing inputs and failure-inducing traces makes exhaustive search prohibitively expensive\linelabel{ln:abs-1-end}.}

This paper develops \fw, a distributional framework \rtk{that addresses this challenge by exponentially tilting} a diffusion model-induced joint distribution over environments and executions.
We show that such diffusion-guided sampling admits an exact interpretation as importance sampling in the joint space, where guidance scores induce a KL-optimal reallocation of probability mass towards failure-relevant behaviors. In addition, we show that tilting provably amplifies the probability of failure and strictly outperforms conditional sampling strategies\rtk{, which are limited by multiplicative rarity effects}. In this framework, the joint generative model serves as a reusable prior over scenarios and need not faithfully represent the system under test. Instead, true system evaluations, by means of expensive simulations, are limited to the learning of a scoring function responsible \rtk{for characterizing} the quality of the produced scenario, thus enabling selective and adaptive use of expensive simulations.
\linelabel{ln:abs-2-start}\rtk{Additionally, while we focus on falsification, the framework is general and applies to other scenario-generation problems such as planning, synthesis\linelabel{ln:abs-2-end}.} 

\rtk{In this work, we }study \fw performance against benchmarks from the ARCH-COMP, and we propose an additional tractor-trailer benchmark for falsification to show the behavior of several approaches when scenario generation is guided by a well-defined specification as opposed to a reward. Across all cases, the proposed method achieves competitive or improved falsification performance compared to state-of-the-art approaches, with larger gains when specification definition is not limited to STL formulas\linelabel{ln:abs-end}.

\end{abstract}

\section{Introduction and Related Literature}\label{sec::intro}

The discovery of rare but safety-critical failures in complex cyber-physical and autonomous systems remains a fundamental challenge due to the high dimensionality, structured constraints, and strong dependencies between environments and system executions. Existing approaches to scenario testing and falsification typically rely on conditional sampling strategies or heuristic search over inputs, which implicitly factor the joint distribution over environments and trajectories, and therefore suffer from multiplicative rarity effects~\cite{dreossi2018semantic,vin20233d,fremont2020formal}. From a probabilistic perspective, such methods sample at the density of rare events rather than actively reshaping it, leading to prohibitive sample complexity even when conditional generative models are accurate.

Importance sampling provides a principled framework for rare-event estimation by reweighting probability mass toward critical regions, with exponential tilting emerging as the KL-optimal mechanism for enforcing moment constraints and variance reduction~\cite{csiszar1975idivergence,cover2006elements,bucklew2004introduction}. These ideas are deeply connected to large deviations theory and Gibbs measures, where exponential reweighting concentrates mass on extremal events~\cite{donsker1975asymptotic,DemboZeitouni1998LDP}. However, classical importance sampling requires access to an ideal reference distribution or an explicit likelihood model, assumptions that are rarely satisfied in realistic autonomous-system testing pipelines.

Recent advances in score-based generative modeling and diffusion processes enable the learning of expressive joint distributions over high-dimensional, structured data, providing a new opportunity to revisit rare-event discovery from a distributional perspective~\cite{song2019generative,Ho2020DDPM,song2021scorebased}. In parallel, classifier- and score-guided diffusion methods have demonstrated that generative processes can be steered toward semantically meaningful regions of the data space via learned guidance signals~\cite{dhariwal2021diffusion}. In this work, we unify these developments by interpreting diffusion-guided falsification as exponential tilting of a learned joint generative model over environments and executions. We show that diffusion guidance induces an exact importance sampling scheme on the joint space, where surrogate robustness scores act as likelihood-ratio tilts that provably amplify failure probability under mild ranking assumptions. 
Finally, we demonstrate that for deterministic dynamical systems, joint diffusion-guided tilting reduces to optimal importance sampling over inputs along the dynamics manifold, yielding a unified theoretical framework for distribution-guided rare-event discovery.

\pbupdate{
While the analysis is motivated by falsification and counterexample generation, the proposed framework is general. Exponential tilting of a joint scenario distribution provides a mechanism for steering sampling toward behaviors of interest beyond safety violations, including synthesis, certificate estimation. From a verification perspective, diffusion-guided tilting enables distribution-aware exploration of complex scenario spaces, rather than being restricted to falsification.
}

\pbupdate{
Another aspect of our approach is the separation between the \emph{generative model} and the \emph{system under test (SUT)}. The generative model defines a prior over environments and executions and serves solely as a base density. Hence, it need not faithfully represent the ``true'' system dynamics. Our theoretical results rely on the properties of the score and the tilting mechanism, rather than the accuracy of the learned model.
}

\pbupdate{In fact, the interaction with the SUT occurs through the score function (e.g., a robustness monitor or runtime oracle), so simulations are invoked selectively while most sampling is performed within the generative model. This decoupling enables efficient exploration, maintains compatibility with black-box verification pipelines, and allows \fw to use any score defined on trajectories, including robustness, reward, or expert-defined costs.}

\paragraph{Contributions}
This paper provides a formal foundation for diffusion-guided falsification as a principled counterexample generation mechanism. The contributions are as follows: 
\begin{enumerate}
    \item We introduce a joint distributional formulation of falsification, in which failure discovery is interpreted as probability mass reallocation, and show that exponential tilting yields the KL-optimal biased distribution.
    \item We establish sufficient and necessary conditions under which exponential tilting provably amplifies failure probability, showing via first-order stochastic dominance 
    (FOSD) 
    that tilting is effective precisely when the score ranks failing behaviors ahead of non-failing ones in distribution.
    \item We characterize a fundamental limitation of conditional sampling, showing that even with perfect conditionals, failure probability remains bottlenecked by multiplicative rarity, whereas joint tilting overcomes this by reallocating mass in the joint space.
    \item For deterministic dynamical systems, we show that joint tilting reduces to optimal importance sampling over inputs along the dynamics manifold, bridging the framework with black-box simulators and existing simulation-based verification workflows. 
    \item We propose the TT2D tractor-trailer benchmark adapted from~\cite{kim2026safempd} as a controlled falsification benchmark, providing an expressive setting in which falsification difficulty can be systematically tuned.
\end{enumerate}

These results make the proposed diffusion-guided falsification \fw a sound, distribution aware enhancement to counterexample generation, grounded in importance sampling theory and stochastic order analysis, and directly relevant to the verification and validation of complex CPS.

\section{Diffusion Guided Search as Importance Sampling (IS)}\label{sec::diffasIS}

We begin by formalizing the class of scenario samplers considered in this work. Rather than committing to a specific algorithm or implementation, we adopt an abstract view in which a generative model defines a base distribution over environments and executions, and guidance is introduced through a scoring function that biases sampling toward behaviors of interest. This abstraction allows us to reason about guided sampling independently of how the underlying generative model is obtained, and to isolate the probabilistic mechanisms that drive failure discovery. In particular, we consider samplers that operate over a joint space of environment parameters and system executions, and that can be guided using likelihood-based reweighting. Diffusion models provide a convenient and expressive instantiation of this class, but the formulation in this section applies more broadly to any generative mechanism that supports sampling from a joint density and evaluation of likelihood ratios. The remainder of this section introduces the notation, assumptions, and sampling primitives that will be used throughout the paper, setting the stage for the formal analysis in Section~\ref{sec::main:analysis}.

\rtk{
\paragraph{Preliminaries}\linelabel{ln:diff-prelim-start}
Let $z_0=\left(x,y\right)\in\mathbb{R}^d$ be a vectorized 
representation of the \rtk{inputs and the corresponding outputs}. A diffusion model learns to generate samples from $p(z_0)$ by defining two coupled Markov chains: a fixed \emph{forward process} that gradually corrupts $z_0$ into Gaussian noise, and a learned \emph{reverse process} that denoises noisy samples and transforms them back into samples from the data distribution \cite{Ho2020DDPM,song2021scorebased}.

\noindent\textbf{Forward process.}\linelabel{ln:ddpm-forward-start}
A diffusion model corrupts a data sample $z_0$ into Gaussian noise
via a fixed Markov chain~\cite{Ho2020DDPM}:
\begin{align}\label{eq:ddpm_forward}
    q(z_t \mid z_{t-1})
    = \mathcal{N}\!\left(\sqrt{\alpha_t}\,z_{t-1},\,
      (1-\alpha_t)I\right),
\end{align}
with marginal $q(z_t \mid z_0)
= \mathcal{N}(\sqrt{\bar\alpha_t}\,z_0,\,(1-\bar\alpha_t)I)$,
where $\bar\alpha_t = \prod_{s=1}^{t}\alpha_s$ \linelabel{ln:ddpm-forward-end}.

\noindent\textbf{Reverse process and sampler.}\linelabel{ln:ddpm-reverse-start}
The reverse process is a learned Markov chain $p_\theta(z_{t-1}\mid z_t)
= \mathcal{N}(\mu_\theta(z_t,t),\,\sigma_t^2 I)$.
A neural network $\varepsilon_\theta(z_t, t)$ is trained to predict
the noise $\varepsilon$ injected at level $t$~\cite{Ho2020DDPM},
recovering the posterior mean:
\begin{align}\label{eq:ddpm_mean}
    \mu_\theta(z_t, t)
    = \frac{1}{\sqrt{\alpha_t}}
      \!\left(z_t
             - \frac{1-\alpha_t}{\sqrt{1-\bar\alpha_t}}\,
               \varepsilon_\theta(z_t, t)\right).
\end{align}
Sampling proceeds as \linelabel{ln:ddpm-reverse-end}
\begin{equation}\label{eq:ddpm_sampler}
    z_T \sim \mathcal{N}(0, I), \quad
    z_{t-1} = \mu_\theta(z_t, t) + \sigma_t \eta,
    \quad \eta \sim \mathcal{N}(0, I).
\end{equation}

\noindent\textbf{Score interpretation and the VP-SDE connection.}\linelabel{ln:reverse-sde-start}
Song~et~al.~\cite{song2021scorebased} show that the discrete forward chain
above is a discretization of the \emph{Variance Preserving SDE}
(VP-SDE):
\begin{align}
    dz = -\tfrac{1}{2}\beta(t)\,z\,dt + \sqrt{\beta(t)}\,dW_t.
\end{align}
Under this continuous view, the reverse of any such diffusion is
governed by the Stein score $\nabla_{z_t}\!\log p_t(z_t)$
via the reverse-time SDE~\cite{song2021scorebased}:
\begin{align}\label{eq:reverse_sde}
    dz = \bigl[f(z,t)
         - g(t)^2\,\nabla_{z}\!\log p_t(z)\bigr]\,dt
         + g(t)\,d\bar{W}_t.
\end{align}
The noise prediction $\varepsilon_\theta$ is related to the score by
\begin{align}\label{eq:score}
    \nabla_{z_t}\!\log p_t(z_t)
    \approx -\frac{\varepsilon_\theta(z_t,t)}{\sqrt{1-\bar\alpha_t}},
\end{align}
and the reverse diffusion process used in DDPMs can be viewed as a discrete approximation of Eq.~\eqref{eq:reverse_sde} under this identification~\cite{song2021scorebased}. This interpretation makes \emph{guidance} natural, since modifying the score steers samples toward desired regions of $p_t$\linelabel{ln:diff-prelim-end}.

\paragraph{Target Sampling Density.}
The sampler above draws from the base joint distribution $p_\theta(x,y)$ with no preference for any region. In many applications one instead wants samples that score highly under a scalar objective $R:\mathcal{Z}\rightarrow\mathbb{R}$, whether $R$ encodes task reward, constraint violation, or adversarial cost. The question is how to tilt the sampler toward high-scoring regions without discarding the prior encoded in $p_\theta$. Without loss of generality we consider minimizing the objective. Let $R(x,y)$ be the target function and $S_k(x,y)$ its surrogate, giving the tilted sampling distribution}
\begin{align}\label{eqn::proposalQ}
p\left(x,y\right) &:= \frac{p_{\theta}\left(x,y\right)\exp\left(\beta S\left(x,y\right)\right)}{Z},\nonumber\\
Z &= \mathbb{E}_{\left(X,Y\right)\sim p_{\theta}}\left[\exp\left(\beta S\left(X,Y\right)\right)\right].
\end{align}

\section{Analysis}\label{sec::main:analysis}

This section develops a distributional analysis of diffusion-guided falsification. We model the generative environment as a joint distribution over environments and executions, and study the effect of exponential tilting on failure probability. Our goal is to characterize when tilting amplifies failures, how this depends on the score function, and why it overcomes the multiplicative rarity barrier of conditional sampling.

The analysis isolates the probabilistic mechanisms underlying guided sampling and establishes guarantees independent of a particular diffusion implementation. These results formalize diffusion-guided falsification as principled importance sampling (IS) and clarify its advantage over conditional approaches.

\paragraph*{Summary of results}\linelabel{ln:sum_start}
\rtk{We cast the efficient discovery of rare failures as IS. A failure requires both a rare input (w.p. $\delta \ll 1$) and a corresponding rare trace (w.p. $\varepsilon \ll 1$), so its probability is the product $\delta\varepsilon$ of these two small probabilities, and the expected number of samples required to observe one is of order $1/(\delta\varepsilon)$, even under an exact conditional model (Proposition~\ref{prop::condBottleneck}). Rather than draw from the model at this cost, we reweight the joint distribution over inputs and traces so that failures are sampled more often.
The reweighting we adopt is exponential tilting by a score $S$, a surrogate that is larger on more failure-like samples, which shifts the joint distribution toward regions where failures are more likely. Proposition~\ref{prop::optTilt} and Lemma~\ref{lem::admisstilt} show that this reweighting is well defined and is the least-distorting one that raises the expected score, remaining closest to the model in KL divergence. 
Proposition~\ref{prop::failAmpl} shows that the probability of sampling a failure under the tilted distribution is an exact importance-sampling reweighting of the original model, and that tilting raises this probability precisely when failures receive higher scores than non-failures. 
Theorem~\ref{thm::faiAmplv2} weakens this requirement to an FOSD assumption, under which tilting still provably increases the failure probability, at the cost of a weaker guarantee: it raises the probability but no longer drives it to one. Corollary~\ref{cor::concsep} recovers that stronger conclusion under strict score separation, and 
Corollary~\ref{cor::condcomp} establishes the advantage over conditional sampling, which is capped by the failure rate within a single input because it cannot reweight traces, a limit the joint tilt overcomes. Notation is summarized in Table~\ref{tab::notation}\linelabel{ln:sum_end}.
}

\begin{table}[]
\centering
\caption{Notation used in Section~\ref{sec::main:analysis}.}
\label{tab::notation}
\renewcommand{\arraystretch}{1.3}
\begin{tabular}{ll}
\toprule
\textbf{Symbol} & \textbf{Meaning} \\
\midrule
$z = (x,y) \in \mathcal{X}\times\mathcal{Y}$ & Joint state; input and execution trace \\
$p^{\star}(x,y)$ & True joint distribution over scenarios \\
$p_{\theta}(x,y)$ & Learned joint diffusion prior \\
$R:\mathcal{X}\times\mathcal{Y}\to\mathbb{R}$ & Cost / robustness function \\
$F := \{(x,y): R(x,y)\le 0\}$ & Failure set \\
$\mathbf{1}_{F}(x,y)$ & Indicator of $F$ \\
$S:\mathcal{X}\times\mathcal{Y}\to\mathbb{R}$ & Score function (surrogate for failure) \\
$A \subset \mathcal{X}$ & Input region where failures concentrate \\
$\delta = p^{\star}(A\times\mathcal{Y})$ & Outer rarity parameter \\
$\varepsilon = p^{\star}(F\mid X\in A)$ & Within-region rarity parameter \\
$\beta \ge 0$ & Tilting parameter \\
$Z(\beta) = \mathbb{E}_{p_{\theta}}[e^{\beta S}]$ & Normalization Constant\\
$p_{\theta,\beta} \propto p_{\theta}e^{\beta S}$ & Tilted distribution \\
\bottomrule
\end{tabular}
\end{table}

\paragraph*{Problem Setup}\label{sec::pSetup}

Let 
$\left(X,Y\right)\in\mathcal{X}\times\mathcal{Y}$ 
denote a scenario, where $X$ is some representation of the input, and $Y$ is a system execution trace consistent with $X$.
Let $p^\star\left(x,y\right)$ denote the true (unknown) distribution of \textit{plausible} inputs. We assume access to a learned joint diffusion model $p_{\theta}\left(x,y\right)$ that approximates $p^\star$ and provides a support to the sampler. We also define a failure event as the set satisfying:
\begin{equation}
    F:=\left\lbrace\left(x,y\right)\in\mathcal{X}\times\mathcal{Y}:R\left(x,y\right)\le 0\right\rbrace,
\end{equation}
where $R:\mathcal{X}\times\mathcal{Y}\rightarrow\mathbb{R}$ is a cost function, a score, associated with the (input, trace) pair. The goal of sampling is to efficiently discover locations from $F$, or equivalently, to estimate or \textit{amplify} the probability mass in regions where $R\left(x,y\right)$ is small. In order to formalize the desired properties for our sampler, we define the concept of \textit{rare failure} as failure event of interest in this work.

\begin{definition}(Rare Failure)\label{def:rare_failure_def}
    A failure event $F$ is rare if there exists a measurable set $A\subset \mathcal{X}$ such that:
    \begin{itemize}
        \item $p^\star\left(A\times\mathcal{Y}\right)=\delta << 1$,
        \item Failures are negligible outside of the set $A$: \[
        p^\star\left(F|X\notin A\right)=\frac{p^\star\left(F\cap \left(\left(\mathcal{X}\setminus A\right)\times\mathcal{Y}\right)\right)}{p^\star\left(\left(\mathcal{X}\setminus A\right)\times\mathcal{Y}\right)}\approx 0.
        \]
        \item Failure is rare but possible within $A$:
        \[p^\star\left(F|X\in A\right)=\frac{p^\star\left(F\cap \left(A\times\mathcal{Y}\right)\right)}{p^\star\left(A\times\mathcal{Y}\right)}=\varepsilon<<1.\]
    \end{itemize}
    Hence, the overall failure probability is $p^\star\left(F\right)=\delta\varepsilon$, \rtk{which we refer to as the \emph{multiplicative rarity} of $F$.}\linelabel{ln:mult_rar_effect}
\end{definition}

\begin{example}[Multiplicative Rarity Effect]\label{ex:tt2d_mult_rarity}
\rtk{\linelabel{ln:eg-mul-rar-start}Consider a tractor-trailer maneuvering task, in which a vehicle composed of a tractor and an attached trailer must drive from an initial configuration to a goal configuration in a bounded planar environment using bounded velocity and steering inputs. The property used for evaluation is the negation of goal reachability: the system is specified to never reach the goal region. Thus, a falsifying event $F$ corresponds to a successful goal-reaching maneuver. Reaching the goal requires (i) that the input trajectory brings the vehicle into the neighborhood of the goal at all, which only a narrow corridor of steering and velocity sequences achieves, so the admissible inputs form a small set $A$ with $\delta = p^{\star}(A \times \mathcal{Y}) \ll 1$; and (ii) that the final configuration simultaneously satisfies the position tolerance and both the tractor and trailer orientation tolerances, which only a fraction of near-goal trajectories meet, giving the within-region rate $\varepsilon \ll 1$. Thus, falsification requires both rare conditions to coincide, so $p^{\star}(F) = \delta\varepsilon$. Tightening the tolerances (smaller $\varepsilon$) or narrowing the feasible corridor (smaller $\delta$) shrinks the failure region multiplicatively. Proposition~\ref{prop::condBottleneck} formalizes this bottleneck\linelabel{ln:eg-mul-rar-end}.}
\end{example}

\begin{assumption}[Score Function Characterization]\label{asm::score}
    The score function is positively correlated with the failure, namely:
    \begin{equation}
        \sup_{\left(x,y\right)\in \mathcal{X}\times\mathcal{Y}\setminus F} S\left(x,y\right) > \sup_{\left(x,y\right)\in  F} S\left(x,y\right).  \label{eqn::Sasm}
    \end{equation}
\end{assumption}

\noindent\emph{Conditional Sampling and the Rare-event bottleneck. } A conditional sampling approach generates scenarios via $X\sim q\left(x\right), Y\sim p_{\theta}\left(y|X\right)$, where $q$ is a proposal distribution over scenes. Under this scheme, the probability to observe a failure is:
\[
    \mathbb{P}\left(F\right)=\int q\left( x\right)p_{\theta}\left(F|x\right)dx.
\]
\begin{proposition}(Conditional sampling expected efficiency)\label{prop::condBottleneck}
    Suppose $p_{\theta}\left(y|x\right)=p^{\star}\left(y|x\right)$, and $q\left(x\right)\approx p^{\star}\left(x\right)$. Then it holds, for the rare failure $F$ that 
        $\mathbb{P}\left(F\right) = \delta\varepsilon,$
    and the expected number of samples required to observe a failure is $\Omega\left(\left(\delta\varepsilon\right)^{-1}\right)$.
\end{proposition}

\begin{proof}
See Appendix~\ref{app:prop_1_proof}.
\end{proof}

Proposition~\ref{prop::condBottleneck} tells us that even with a perfect conditional model, a conditional sampling distribution must first draw a rare input $X\in A$ and then a rare trace $Y\in B_X$, which implies a multiplicative effect. Importantly, improving $p_{\theta}\left(y|x\right)$ does not mitigate this effect, unless the proposal $q\left(x\right)$ already concentrates on the unknown failure inducing region $A$. Another intuitive characterization of the conditional sampling is that such method wants to identify failure prone inputs $x$ and then sample $y\sim p_\theta\left(y|x\right)$ and evaluate $R\left(x,y\right)$ to achieve an estimate for $p\left(F|X=x\right)=\mathbb{P}_{Y\sim p_{\theta}\left(\cdot|x\right)}\left(R\left(x,y\right)\le 0\right)$. 

Under the base distribution $p^{\star}\left(X,Y\right)$, the probability of a failure event is:
\begin{equation}
    p^{\star}\left(F\right)=\mathbb{E}_{p^{\star}}\left[\boldsymbol{1}_{F}\left(X,Y\right)\right].\label{eqn::probFailExact}
\end{equation}
An intuitive IS distribution for the estimation of $p^{\star}\left(F\right)$ would then be:
\[
    p^{\star}\left(x,y\mid F\right)\propto p^{\star}\left(x,y\right)\boldsymbol{1}_{F}\left(x,y\right).
\]
This distribution places all mass on the failure set and minimizes the variance associated with the estimator. However, it is infeasible to sample from this IS proposal because $F$ is unknown a priori. A standard relaxation is to replace the indicator $\boldsymbol{1}_{F}\left(x,y\right)$ with a smooth score $S\left(x,y\right)$ correlated with failure. Now let our score function $S:\mathcal{X}\times \mathcal{Y}\rightarrow \mathbb{R}$ be interpreted as a surrogate for failure (as an example, this would be the learned robustness predictor), then the following holds for the tilted distribution defined in equation~\eqref{eqn::proposalQ}.

\begin{proposition}[Variational optimality for the tilted distribution]\label{prop::optTilt}
    Define the \rtk{tilted} distribution 
    \begin{align*}
        p_{\theta,\beta}\left(x,y\right) &:= \frac{p_{\theta}\left(x,y\right)\exp\left(\beta S\left(x,y\right)\right)}{Z\left(\beta\right)},
    \end{align*}
    where $Z\left(\beta\right) = \mathbb{E}_{\left(X,Y\right)\sim p_{\theta}}\left[\exp\left(\beta S\left(X,Y\right)\right)\right],$ then $p_{\theta,\beta}$ is the solution to
    \begin{align*}
        \min_{q} KL\left(q\mid\mid p_{\theta}\right)\quad \mbox{s.t. } \mathbb{E}_{\left(X,Y\right)\sim q}\left[S\left(X,Y\right)\right]\ge c
    \end{align*}
    for some $c$ function of $\beta$.
\end{proposition}



\rtk{\begin{proof}
Fix\linelabel{ln:prop-2-proof-start} $\beta\in\mathcal B$ so that $Z(\beta)<\infty$ and $p_{\theta,\beta}$ is well-defined.
Let $\mathcal Q:=\{q:\ q\ll p_\theta,\ \int q(z)\,dz=1\}$.
For $q\in\mathcal Q$, define the likelihood ratio $r(z):=\frac{q(z)}{p_\theta(z)}$ so that
$q(z)=r(z)p_\theta(z)$, $r\ge 0$, and $\mathbb E_{p_\theta}[r]=1$.

We first express KL as a convex functional of $r$. Specifically, we have
\begin{align*}
\mathrm{KL}(q\|p_\theta)
&=\int q(z)\log\frac{q(z)}{p_\theta(z)}\,dz\\
&=\int p_\theta(z)\, r(z)\log r(z)\,dz
=\mathbb E_{p_\theta}[r\log r].
\end{align*}
Moreover, $\mathbb E_q[S]=\int q(z)S(z)\,dz=\mathbb E_{p_\theta}[rS]$.
Now consider the constrained problem
\[
\min_{r\ge 0}\ \mathbb E_{p_\theta}[r\log r]
\quad\text{s.t.}\quad
\mathbb E_{p_\theta}[r]=1,\ \ \mathbb E_{p_\theta}[rS]\ge c,
\]
for some scalar $c$. Then the Lagrangian with multipliers $\lambda\in\mathbb R,\beta'\ge 0$ (normalization, and moment constraint respectively) can be written as:
\[
\mathcal L(r,\lambda,\beta')
=
\mathbb E_{p_\theta}[r\log r]
+\lambda(\mathbb E_{p_\theta}[r]-1)
-\beta'(\mathbb E_{p_\theta}[rS]-c).
\]
A first-order optimality condition (calculus of variations) yields, for $p_\theta$-a.e.\ $z$,
\[
\frac{\delta \mathcal L}{\delta r}(z)
=\log r(z)+1+\lambda-\beta' S(z)=0,
\]
hence
\[
r^\star(z)=\exp(\beta' S(z)-1-\lambda)=C\,e^{\beta' S(z)}
\]
for some constant $C>0$.

We can see that setting $\mathbb E_{p_\theta}[r^\star]=1$ gives
\[
1=\mathbb E_{p_\theta}[r^\star]=C\,\mathbb E_{p_\theta}[e^{\beta'S(Z)}]
= C\, Z(\beta'),
\]
so $C=1/Z(\beta')$ and therefore
\[
q^\star_{\beta'}(z)=r^\star(z)p_\theta(z)=\frac{p_\theta(z)e^{\beta' S(z)}}{Z(\beta')}.
\]

Set $\beta'=\beta$ and define $c(\beta):=\mathbb E_{p_{\theta,\beta}}[S]$.
Then $q^\star_{\beta}=p_{\theta,\beta}$ is feasible with equality:
$\mathbb E_{p_{\theta,\beta}}[S]=c(\beta)$.
By convexity of $r\mapsto r\log r$ and linearity of the constraints, the problem is convex;
hence the KKT conditions are sufficient and the optimizer is unique.
Therefore $p_{\theta,\beta}$ is the unique solution of
\[
    \min_q \mathrm{KL}(q\|p_\theta)\ \text{s.t.}\ \mathbb E_q[S]\ge c(\beta).
\]
This completes the proof\linelabel{ln:prop-2-proof-end}.
\end{proof}}

Proposition \ref{prop::optTilt} shows how tilting produces the closest distribution to the data (in KL divergence) that amplifies failure-relevant regions.
\begin{lemma}\label{lem::admisstilt}
    Let $p$ be a density defined over $\mathcal{Z}=\mathcal{X}\times \mathcal{Y}$, and let $S:\mathcal{Z}\rightarrow \mathbb{R}$ be a measurable score. Let $Z\left(\beta\right)$ be defined as $Z\left(\beta\right):=\int p\left(z\right)\exp\left(\beta S\left(z\right)\right)$, then an admissible tilt set is
    \begin{equation}
        \mathcal{B} = \left\lbrace \beta : Z\left(\beta\right)<\infty\right\rbrace.
    \end{equation}
    The following properties hold for $\mathcal{B}$:
    \begin{enumerate}
        \item $0\in\mathcal{B}, Z\left(0\right)=1$;
        \item $\mathcal{B}=\left(\beta_{-},\beta_{+}\right)$, with $-\infty\le\beta_{-}\le 0\le\beta_{+}\le +\infty$;
        \item For any $\beta\in\mathcal{B}$ the tilted density $p_{\beta}\left(z\right):=\frac{p\left(z\right)\exp\left(\beta S\left(z\right)\right)}{Z\left(\beta\right)}$ is a valid density function over $\mathcal{Z}$.
    \end{enumerate}
    Moreover, with $\beta\in\mbox{int}\left(\mathcal{B}\right)$, assume that $\int p\left(z\right)\exp\left(\left(\beta + \delta\right) S\left(z\right)\right)<\infty$ for some $\delta > 0$, then the following holds:
    \begin{enumerate}
        \item $\psi\left(\beta\right)=\log Z\left(\beta\right)$ is finite and differentiable, and $\psi'\left(\beta\right)=\mathbb{E}_{p_{\beta}}\left[S\left(Z\right)\right]$.
        \item If $\int p\left(z\right)\exp\left(\left(\beta + \delta\right) |S\left(z\right)|\right)<\infty$ for some $\delta > 0$, then $\psi$ is twice differentiable and $\psi''\left(\beta\right)=\mbox{Var}\left(S\left(Z\right)\right)\ge 0$.
    \end{enumerate}
\end{lemma}

\begin{proof}
See Appendix~\ref{app:lemma_1_proof}.
\end{proof}

After characterizing the tilt we analyze its effect on the failure probability.
\begin{proposition}[Effect of tilting on Failure Probability]\label{prop::failAmpl}
    Let $p\equiv p_{\theta}\left(x,y\right)$ be the base joint density on support $\mathcal{Z}=\mathcal{X}\times \mathcal{Y}$, and let $F\subset \mathcal{Z}$ be the failure set defined as $F=\left\lbrace\left(x,y\right):R\left(x,y\right)\le 0\right\rbrace$. Let $S:\mathcal{Z}\rightarrow \mathbb{R}$ be a measurable score, and assume $Z\left(\beta\right):=\int p\left(z\right)\exp\left(\beta S\left(z\right)\right)<\infty$ for admissible $\beta$ (Lemma~\ref{lem::admisstilt}). For the tilted density $p_{\beta}\left(z\right):=\frac{p\left(z\right)\exp\left({\beta S\left(z\right)}\right)}{Z\left(\beta\right)}$, the following properties hold:
    \begin{enumerate}
        \item[(A)] $\int_{F} p_{\beta}\left(z\right)dz = \frac{\int_{F} p\left(z\right)\exp\left({\beta S\left(z\right)}\right)dz}{\int p\left(z\right)\exp\left({\beta S\left(z\right)}\right)dz}$.
        \item[(B)] Let $M_F\left(\beta\right):=\mathbb{E}_p\left[e^{\beta S\left(z\right)}\mid Z\in F\right]$, $M_{\bar{F}}\left(\beta\right):=\mathbb{E}_p\left[e^{\beta S\left(z\right)}\mid Z\notin F\right]$, then the following must hold:
        \begin{equation}
            \int_F p_{\beta}>\int_F p \iff M_{F}\left(\beta\right)>M_{\bar{F}}\left(\beta\right)
        \end{equation}
        \item[(C)] If there exists $m > 0$ such that $\inf_{z\in F}S\left(z\right)\ge \sup_{z\notin F}S\left(z\right)+m$, then for all $\beta\ge 0$, the following bound over $\mathbb{P}_{\beta}\left(F\right)$ can be derived:
        \begin{equation}
            \int_{F} p_{\beta}\left(z\right) dz\ge\frac{1}{1+\frac{1-\int_F p}{\int_F p}e^{-\beta m}},
        \end{equation} and as a result, $\int_F p_{\beta}\left(z\right)\rightarrow 1$ as $\beta\rightarrow\infty$.
    \end{enumerate}
\end{proposition}

\begin{proof}
(A) By definition of $p_\beta$,
\begin{align*}
\int_F p_\beta(z)\,dz
=
\int_F \frac{p(z)e^{\beta S(z)}}{Z(\beta)}\,dz
=
\frac{\int_F p(z)e^{\beta S(z)}\,dz}{\int_{\mathcal Z} p(z)e^{\beta S(z)}\,dz}.
\end{align*}

(B) Let $p(F):=\int_F p(z)\,dz\in(0,1)$. Using the law of total expectation,
\begin{align*}
\int_F p(z)e^{\beta S(z)}\,dz
=
p(F)\,\mathbb E_p[e^{\beta S(Z)}\mid Z\in F]
=
p(F)\,M_F(\beta),
\end{align*}
and similarly,
\begin{align*}
\int_{\bar F} p(z)e^{\beta S(z)}\,dz
=
(1-p(F))\,M_{\bar F}(\beta).
\end{align*}
Substituting into (A) yields
\begin{align*}
\int_F p_\beta
=
\frac{p(F)M_F(\beta)}{p(F)M_F(\beta)+(1-p(F))M_{\bar F}(\beta)}.
\end{align*}
A direct comparison with $p(F)$ gives
\begin{align*}
\int_F p_\beta > p(F)
&\iff
p(F)M_F(\beta) \\ 
&\quad> p(F)\big(p(F)M_F(\beta)+(1-p(F))M_{\bar F}(\beta)\big),
\end{align*}
and
\begin{align*}
    p(F)M_F(\beta) &> p(F)\big(p(F)M_F(\beta)+(1-p(F))M_{\bar F}(\beta)\big)\\
&\iff
M_F(\beta) > M_{\bar F}(\beta).
\end{align*}

(C) Let $a:=\inf_{z\in F}S(z)$ and $b:=\sup_{z\notin F}S(z)$, so $a\ge b+m$.
Then
\begin{align*}
\int_F p(z)e^{\beta S(z)}\,dz \ge e^{\beta a}\int_F p(z)\,dz = p(F)e^{\beta a},
\end{align*}
and
\begin{align*}
\int_{\bar F} p(z)e^{\beta S(z)}\,dz \le e^{\beta b}\int_{\bar F} p(z)\,dz = (1-p(F))e^{\beta b}.
\end{align*}
Using (A) and these bounds,
\begin{align*}
\int_F p_\beta
&=
\frac{\int_F p e^{\beta S}}{\int_F p e^{\beta S}+\int_{\bar F} p e^{\beta S}}\\
&\ge
\frac{p(F)e^{\beta a}}{p(F)e^{\beta a}+(1-p(F))e^{\beta b}}
=
\frac{1}{1+\frac{1-p(F)}{p(F)}e^{-\beta(a-b)}}.
\end{align*}
Since $a-b\ge m$, the stated bound follows. The limit as $\beta\to\infty$ follows.
\end{proof}

Result (A) expresses the failure probability under tilting as an exact likelihood-ratio reweighting of the base distribution, making explicit the connection to IS. Result (B) shows that tilting improves failure discovery if and only if the surrogate score assigns larger exponential weight to failure samples than to non-failures, which is in fact a way to express the notion of score alignment. Result (C) characterizes an idealized regime in which the score perfectly separates failures, in which case increasing the tilt parameter concentrates probability mass arbitrarily close to the failure set. Together, these results clarify both the mechanism and the limits of diffusion-guided falsification.

What we want to show now is that the tilted distribution moves mass towards the failure region. In order to show this we (i) write the tilted failure probability in a form that exposes a likelihood ratio; (ii)  state sufficient conditions under which it is provably larger than the baseline $p_{\theta}\left(F\right)=\delta\varepsilon$. As shown in Proposition~\ref{prop::failAmpl}, we cannot prove this in general, but we need an assumption on the alignment between $S$ and failure. Let the base distribution be $p\equiv p_\theta\left(x,y\right)$, and let the tilted distribution be:
\begin{equation*}
    p_{\beta}\left(z\right)=\frac{p\left(z\right)e^{\beta S\left(z\right)}}{Z\left(\beta\right)},\; z=\left(x,y\right),\;Z\left(\beta\right)=\mathbb{E}_{p}\left[e^{\beta S\left(Z\right)}\right].
\end{equation*}
Let $F=\left\lbrace z:R\left(z\right)\le 0\right\rbrace$. Then:
\begin{align*}
    p_{\beta}\left(F\right)&=\frac{\mathbb{E}_{p}\left[\boldsymbol{1}_{F}e^{\beta S\left(Z\right)}\right]}{\mathbb{E}_{p}\left[e^{\beta S\left(Z\right)}\right]}\\
    &=\frac{p\left(F\right)M_{F}\left(\beta\right)}{p\left(F\right)M_{F}\left(\beta\right)+\left(1-p\left(F\right)\right)M_{\bar{F}}\left(\beta\right)},
\end{align*}
where
\[
    M_{F}\left(\beta\right)= \mathbb{E}_{p}\left[e^{\beta S\left(Z\right)}\mid F\right],\; M_{\bar{F}}\left(\beta\right)= \mathbb{E}_{p}\left[e^{\beta S\left(Z\right)}\mid \bar{F}\right].
\]
We want to show that $p_{\beta}\left(F\right)>p\left(F\right)=\delta\varepsilon$.
\begin{theorem}[Failure Amplification Under Exponential Tilting]\label{thm::faiAmplv2}
    Let $p$ be a probability density on $\mathcal{Z}=\mathcal{X}\times\mathcal{Y}$. Let $F\subset \mathcal{Z}$ be a measurable failure set, and let $S: \mathcal{Z}\rightarrow\mathbb{R}$ be a  measurable score function. Define, for any $\beta\in\mathcal{B}$ (see Lemma~\ref{lem::admisstilt}), with $Z\left(\beta\right):=\int_{\mathcal{Z}}p\left(z\right)e^{\beta S}\left(z\right)dz$, the tilted density
    \[
        p_{\beta}\left(z\right)=\frac{p\left(z\right)e^{\beta S\left(z\right)}}{Z\left(\beta\right)}.
    \]
    Assume that the conditional distribution of $S\left(Z\right)$ given $z\in F$ first-order stochastically dominates the conditional distribution of $S\left(Z\right)$ given $Z\notin F$, namely
    \[
        \mathbb{E}_p\left[\phi\left(S\left(Z\right)\right)\mid Z\in F\right]\ge \mathbb{E}_p\left[\phi\left(S\left(Z\right)\right)\mid Z\notin F\right],
    \]
    for all bounded increasing $\phi\left(\cdot\right)$. Then, for all $\beta\in\mathcal{B}$,
    \[
        \int_F p_{\beta}\left(z\right)dz \ge \int_F p\left(z\right)dz,
    \]
    with strict inequality holding for any $\beta>0$ unless the conditional distributions
of $S(Z)$ given $Z\in F$ and given $Z\notin F$ coincide.
\end{theorem}

\begin{proof}
Fix any $\beta\in\mathcal B$ so that $Z(\beta)<\infty$ and $p_\beta$ is well-defined.
Let $p(F):=\int_F p(z)\,dz$. If $p(F)\in\{0,1\}$, the claim is trivial, so assume $p(F)\in(0,1)$.

By Proposition~\ref{prop::failAmpl}(A) and the law of total expectation,
\begin{align*}
\int_F p_\beta(z)\,dz
&=\frac{\int_F p(z)e^{\beta S(z)}\,dz}{\int_{\mathcal Z} p(z)e^{\beta S(z)}\,dz}\\
&=\frac{p(F)\,\mathbb E_p\!\left[e^{\beta S(Z)}\mid Z\in F\right]}
{\left(\splitfrac{p(F)\,\mathbb E_p\!\left[e^{\beta S(Z)}\mid Z\in F\right]
+}{(1-p(F))\,\mathbb E_p\!\left[e^{\beta S(Z)}\mid Z\notin F\right]}\right)}.
\end{align*}
Define the moment generating function
\begin{align*}
    M_F(\beta)&:=\mathbb E_p\!\left[e^{\beta S(Z)}\mid Z\in F\right],\\
    M_{\bar F}(\beta)&:=\mathbb E_p\!\left[e^{\beta S(Z)}\mid Z\notin F\right].    
\end{align*}

Then
\begin{equation}\label{eq:pbetaf}
\int_F p_\beta(z)\,dz
=
\frac{p(F)\,M_F(\beta)}{p(F)\,M_F(\beta)+(1-p(F))\,M_{\bar F}(\beta)}.
\end{equation}

By assumption, the conditional distribution of $S(Z)$ given $Z\in F$ first-order stochastically dominates
that given $Z\notin F$, equivalently, for every bounded non-decreasing function $\phi$,
\[
\mathbb E_p[\phi(S(Z))\mid Z\in F]\ \ge\ \mathbb E_p[\phi(S(Z))\mid Z\notin F].
\]
Given a value of $\beta>0$, the function $\phi_\beta(s):=\min\{e^{\beta s},K\}$ is bounded and increasing in $K>0$,
hence
\[
\mathbb E_p[\phi_\beta(S(Z))\mid Z\in F]\ \ge\ \mathbb E_p[\phi_\beta(S(Z))\mid Z\notin F].
\]
Letting $K\to\infty$ and applying monotone convergence yields
\[
\mathbb E_p[e^{\beta S(Z)}\mid Z\in F]\ \ge\ \mathbb E_p[e^{\beta S(Z)}\mid Z\notin F],
\]
i.e., $M_F(\beta)\ge M_{\bar F}(\beta)$ for all $\beta>0$. For $\beta=0$ the two are equal to $1$.

Substituting $M_F(\beta)\ge M_{\bar F}(\beta)$ into \eqref{eq:pbetaf} and comparing with $p(F)$ gives
\[
\int_F p_\beta(z)\,dz \ \ge\ p(F)=\int_F p(z)\,dz.
\]
Indeed, the mapping $a\mapsto \frac{p(F)a}{p(F)a+(1-p(F))b}$ is nondecreasing in $a$ for fixed $b>0$.
Assume $\beta>0$ and $p(F)\in(0,1)$. From \eqref{eq:pbetaf},
\[
\int_F p_\beta(z)\,dz > p(F)
\quad\Longleftrightarrow\quad
M_F(\beta) > M_{\bar F}(\beta).
\]
Under FOSD, equality of expectations for all bounded increasing $\phi$
can occur only if the two conditional distributions of $S(Z)$ (given $F$ and given $\bar F$) are identical.
In that case, for $\beta>0$ we have $M_F(\beta)=M_{\bar F}(\beta)$ and no strict amplification occurs.

Conversely, if the conditional distributions differ (i.e., $S(Z)$ is not $p$-a.s.\ equal on $F\cup\bar F$),
then by the defining property of first-order dominance there exists a bounded increasing $\phi$ such that
\[
\mathbb E_p[\phi(S(Z))\mid Z\in F] > \mathbb E_p[\phi(S(Z))\mid Z\notin F],
\]
which in particular implies strict inequality for the increasing function $s\mapsto e^{\beta s}$ (via the truncation
argument above), hence $M_F(\beta)>M_{\bar F}(\beta)$ and therefore $\int_F p_\beta > \int_F p$.

This proves the stated strictness condition.
\end{proof}

Theorem~\ref{thm::faiAmplv2} formalizes when diffusion-guided tilting helps. Tilting improves failure probability if and only if failures tend to receive higher scores than non-failures. The FOSD assumption is purely ranking-based: it does not require perfect separation or calibration of $S$ (as Proposition~\ref{prop::failAmpl} does). Algorithmically, this means that if the surrogate score $S$ orders failing scenarios ahead of non-failing ones in distribution, exponential tilting is guaranteed to increase the rate at which failures are sampled.

\begin{corollary}[Asymptotic concentration under score separation]\label{cor::concsep}
    In addition to the assumptions underlying Theorem~\ref{thm::faiAmplv2}, suppose additionally that:
    \begin{enumerate}
        \item[(i)] The set $\mathcal{B}$ is bounded from above, i.e., $\mathcal{B}\subset[0,\infty)$;
        \item[(ii)] There exists a strict separation margin, i.e., $\inf_{z\in F}S\left(z\right)>\inf_{z\notin F}S\left(z\right)$.
    \end{enumerate}
    Then
    \begin{equation*}
        \lim_{\beta\rightarrow\infty}\int_Fp_{\beta}\left(z\right)dz=1.
    \end{equation*}
    Moreover, letting $m:=\inf_{z\in F}S\left(z\right)-\inf_{z\notin F}S\left(z\right)>0$:
    \begin{equation*}
        \int_Fp_{\beta}\left(z\right)dz\ge\frac{1}{1+\frac{1-\int_F p}{\int_F p}e^{-\beta m}}.
    \end{equation*}
\end{corollary}
Corollary~\ref{cor::concsep} characterizes an idealized limit. If the score perfectly separates failures from non-failures, tilting eventually samples only failures. The explicit bound shows how fast concentration occurs as a function of the margin $m$, the tilt parameter $\beta$, and the baseline failure probability $\int_F$. Conceptually, the algorithm itself imposes no barrier to failure discovery; only the quality of the score $S$ limits performance.
\begin{remark}
    The margin condition in Corollary~\ref{cor::concsep} is not required in practice and serves only to characterize the limiting behavior of exponential tilting. Theorem~\ref{thm::faiAmplv2} applies under substantially weaker ranking assumptions and explains the empirical gains observed at moderate values of $\beta$.
\end{remark}

\begin{corollary}[Strict advantage over conditional proposals]\label{cor::condcomp}
    Let $p$ be a base joint density on $\mathcal{Z} = \mathcal{X}\times\mathcal{Y}$. Let $F\subset \mathcal{Z}$ be a measurable failure set and define the conditional family of proposals:
    \begin{equation*}
        \mathcal{Q}_{\mbox{\tiny{cond}}}:=\left\lbrace q\left(x,y\right)=r\left(x\right)p\left(y\mid x\right):r\mbox{ any density on }\mathcal{X}\right\rbrace,
    \end{equation*}
    such family is composed of densities that may change the marginal over $x$ but keep the sample conditional as the base distribution. Define the conditional failure rate as:
    \begin{equation*}
        f\left(x\right):=\int_{\mathcal{Y}}\boldsymbol{1}_{F}\left(x,y\right)p\left(y\mid x\right)dy.
    \end{equation*}
    Assume there exists a measurable set $B\subset\mathcal{X}$ such that:
    \begin{enumerate}
        \item[(i)] $p\left(B\right)>0$, i.e., the base distribution places positive mass on $B$, and
        \item[(ii)] $f\left(x\right)\le 1-\eta$ for all $x\in B$, for some $\eta>0$.
    \end{enumerate}
    Then the following holds:
    \begin{enumerate}
        \item[(R1)] For every $q\in \mathcal{Q}_{\mbox{\tiny{cond}}}$ with supp$\left(r\right)\subseteq B$,
            $\int_F q\left(x,y\right)dxdy\le 1-\eta.$
        Equivalently, even the best conditional sampler supported on $B$ cannot make failure probability exceed $1-\eta$.
        \item[(R2)] Suppose there exists a measurable score $S:\mathcal{Z}\rightarrow\mathbb{R}$ such that
            $\inf_{z\in F}S\left(z\right)>\inf_{z\notin F}S\left(z\right),$
        and suppose the exponential tilt is admissible for all $\beta\ge 0$ (i.e., $Z\left(\beta\right)=\int p\left(z\right)e^{\beta S\left(z\right)}dz<\infty$ for all $\beta\ge 0$, the tilted distribution $p_{\beta}\left(z\right)=\frac{p\left(z\right)e^{\beta S\left(z\right)}}{Z\left(\beta\right)}$ satisfies:
            $\lim_{\beta\rightarrow\infty}\int_Fp_{\beta}\left(z\right)dz=1.$
    \end{enumerate}
\end{corollary}
R1 is the core limitation: if failure is not guaranteed given 
$x$ (i.e., $f\left(x\right)<1$), then any conditional proposal that keeps $p\left(y\mid x\right)$ fixed cannot force failure probability to $1$. Instead, it is capped by the within-$x$ failure rate. R2 shows the contrasting capability: a joint proposal (realized by tilting / guided diffusion) can re-weight within the fiber over each $x$ and concentrate on failing $\left(x,y\right)$ pairs.

\begin{proof}
We prove each result individually.

\noindent(R1) For $q\left(x,y\right) = r\left(x\right)p\left(y\mid x\right)$ (definition of reward and conditional density):
    \begin{align*}
        \int_F q &= \int_{\mathcal{X}}r\left(x\right)\left(\int\boldsymbol{1}_{F}\left(x,y\right)p\left(y\mid x\right)dy\right)dx\\
        &=\int r\left(x\right)f\left(x\right)dx.
    \end{align*}
(R2) This is exactly Corollary~\ref{cor::concsep} from before (score separation implies concentration) applied to the joint space.
\end{proof}

Together, these results show that diffusion-guided tilting overcomes the multiplicative rarity barrier faced by conditional sampling. In the next section, we discuss how this viewpoint motivates the design of our guided sampling procedure.

\begin{remark}[On model mismatch]
While our results establish failure amplification under the learned model $p_\theta$, mismatch with the true distribution $p^\star$ remains. Exponential tilting reweights samples by $e^{\beta S(z)}$, so discrepancies in high-score regions are amplified as $\beta$ increases. This induces a trade-off: larger $\beta$ improves failure discovery but may also magnify model error. A precise characterization of this effect is left for future work.
\end{remark}

\section{Implications for Falsification and Verification}\label{sec::model-guidance}

The analysis developed in the previous section establishes diffusion-guided sampling as a form of distributional importance sampling with provable failure amplification properties. In this section, 
we interpret exponential tilting as a principled mechanism for counterexample generation, and clarify how it overcomes fundamental limitations of conditional sampling strategies commonly used in simulation-based testing.

Our goal is to ground the proposed framework in established verification practice, showing how diffusion-guided tilting complements existing falsification pipelines by providing distribution-aware guidance that remains compatible with black-box simulators, runtime monitors, and deterministic system dynamics. This perspective positions diffusion-guided falsification not as a replacement for formal verification techniques, but as a theoretically justified enhancement to simulation-based counterexample search in complex, high-dimensional systems.

To connect this formulation to practical falsification pipelines, it is important to clarify the role of the deterministic simulator.
\tk{In many falsification pipelines, the execution trace $y$ is not generated by a learned model of the dynamics (such as a diffusion model), but rather by a simulator. In this setting, the method generates a scenario specification $x \in \mathcal{X}$, with $x \sim q(x)$, and the corresponding system execution is given by $y = \Phi(x)$, where $\Phi$ denotes the (possibly deterministic) system dynamics. In this case, the failure set can be equivalently defined in the input space as}
\[
    F = \left\lbrace x\in\mathcal{X}:R\left(x,\Phi\left(x\right)\right)\le 0\right\rbrace.
\]
Thus, failure becomes an event in $\mathcal{X}$, and the rarity is now:
\[
    p\left(F\right) = \mathbb{P}_{X\sim q}\left(X\in F\right),
\]
with an associated sample complexity $\Theta\left(1/p\left(F\right)\right)$ regardless of whether the simulator is deterministic or stochastic. Given these considerations, we can interpret the role of joint modeling through the following remark.

\begin{remark}[Deterministic Dynamics] \label{rem:detdyn}
    If $y=\Phi\left(x\right)$, then the true data lives on a manifold:
\[
\mathcal{M} = \left\lbrace\left(x,y\right):y=\Phi\left(x\right)\right\rbrace\subset \mathcal{X}\times \mathcal{Y},
\]
and a joint diffusion is learning a density concentrated near $\mathcal{M}$. Tilting the joint density by a score $S\left(x,y\right)$ is then equivalent to tilting along the manifold, namely:
\[
    p_k\left(x\right)\propto p_{\theta}\left(x\right)\exp\left(\beta S\left(x,\Phi\left(x\right)\right)\right).
\]
Equivalently, joint tilting reduces to optimal importance sampling over $x$.
\end{remark}

We now adapt the key result from Theorem~\ref{thm::faiAmplv2} to the case where the physical system is deterministic (unique rollout), while the diffusion model may still be stochastic as a generator. We start by assuming that there exist\rtk{s} a measurable function (simulator/dynamics)
    $\Phi:\mathcal{X}\rightarrow\mathcal{Y},\; y=\Phi\left(x\right),$
and define the failure set in input space as
\begin{align*}
   F_{X}:=\left\lbrace x\in \mathcal{X}:\left(x,\Phi\left(x\right)\right)\in F\right\rbrace=\left\lbrace x:R\left(x,\Phi\left(x\right)\right)\le 0\right\rbrace. 
\end{align*}

\begin{proposition}[Joint tilting reduces to input-space tilting on the dynamics manifold]\label{prop::dynjoint}
    Assume the base joint distribution is supported on (or concentrates on) the dynamics manifold $\mathcal{M}:=\left\lbrace\left(x,y\right):y=\Phi\left(x\right)\right\rbrace$. Then the joint density can be rewritten as:
    \begin{equation*}
        p\left(x,y\right)=p_{X}\delta\left(y-\Phi\left(x\right)\right),
    \end{equation*}
    where $p_{X}$ is the induced marginal over inputs. 
    
    Define the joint tilt with score $S\left(x,y\right)$, and the tilted density
    \begin{equation}\label{eq:tilted_density}
        p_{\beta}\left(x,y\right)\propto p\left(x,y\right)e^{\beta S\left(x,y\right)}.
    \end{equation}
    Then, the induced marginal over inputs under the tilted joint satisfies:
    \begin{equation*}
        p_{\beta,X}\left(x\right):=\int_{\mathcal{Y}}p_{\beta}\left(x,y\right)dy\propto p_X\left(x\right)\exp\left(\beta S\left(x,\Phi\left(x\right)\right)\right).
    \end{equation*}
\end{proposition}
\begin{proof}
    See Appendix~\ref{app:prop_4_proof}.
\end{proof}
\tk{In the deterministic case, the ``joint'' method reduces to an importance sampler over inputs where the weights are evaluated along the realized trajectory $y=\Phi\left(x\right)$. In other words, the proposed mechanism composed of the joint diffusion and tilting provides a principled way to learn a proposal over $x$ that respects feasibility ($y=\Phi\left(x\right)$) automatically.}

\begin{corollary}[Deterministic comparison: diffusion-guided tilting vs input-only search]\label{cor::detcomXsearch}
    Let $F_X=\left\lbrace x:R\left(x,\Phi\left(x\right)\right)\le 0\right\rbrace$ and suppose there exists an input score $T\left(x\right):=S\left(x,\Phi\left(x\right)\right)$ that separates failures and non-failures in input space:
    \begin{equation*}
        \inf_{x\in F_X} T\left(x\right)>\sup_{x\notin F_X} T\left(x\right), 
    \end{equation*}
    and the input space partition function is finite for all $\beta \ge 0$:
    \begin{equation*}
        Z_X\left(\beta\right):=\int p_X\left(x\right)e^{\beta T\left(x\right)}dx <\infty.
    \end{equation*}
    Define the tilted input proposal
        $p_{\beta,X}\left(x\right)=\frac{p_X\left(x\right)e^{\beta T\left(x\right)}}{Z_X\left(\beta\right)}.$
    Then the tilted input distribution satisfies:
        $\lim_{\beta\rightarrow\infty}\int_{F_X} p_{\beta,X}\left(x\right)dx =1$.
\end{corollary}

Hence, even when the underlying system is deterministic, tilting yields an optimal ``distributional'' way to push probability mass onto failing inputs (subject to staying close to $p_X$, i.e., realism). This is the deterministic analogue of failure concentration in joint space.

\rtk{\linelabel{ln:joint-advantages-start}Proposition~\ref{prop::dynjoint} and Corollary~\ref{cor::detcomXsearch} show that under deterministic dynamics, joint and input-space tilting target the same distribution. They still differ in how the prior is learned, how the search is guided, and how validity is ensured, and each favors the joint model. First, fitting the joint $p_\theta(x,y)$ to inputs and traces together lets the traces constrain the input marginal $p_\theta(x)$, reducing its epistemic uncertainty. Second, the search is guided without involving the simulator. For an input-space score, the gradient with respect to the input must be propagated through $\Phi$ and is unavailable when $\Phi$ is a black box, whereas the joint model exposes it directly in joint space from its own predicted trace. This joint-space gradient is approximate, but it never differentiates through $\Phi$.\linelabel{ln:joint-advantages-end} \linelabel{ln:phys-valid-start}Third, validity holds regardless of model accuracy. Only the candidate input $x$ is sent to the simulator, which returns the true execution $\Phi(x)$, so every counterexample is valid by construction even if the predicted trace $\hat y$ is wrong. Their gap measures the model's error on that sample and can refine $p_\theta$ online (Section~\ref{sec::conclusion})\linelabel{ln:phys-valid-end}}.


\noindent\emph{Illustrative Example.}
To illustrate the deterministic interpretation, consider a low-dimensional setting with a nonlinear map $y=\Phi(x)$ for $x,y\in\mathbb{R}^2$. We define the robustness as the Himmelblau's function
and define the score as $S(x,y) = -R(y)$, inducing the input-space score $S(\Phi(x))$.

Fig.~\ref{fig:running_example} shows the map $x \mapsto y=\Phi(x)$ together with the induced robustness landscape $R(\Phi(x))$ over the input space, which defines the failure-relevant regions. Fig.~\ref{fig:tilting_det} illustrates the effect of exponential tilting for increasing $\beta$. As $\beta$ increases, probability mass is progressively reallocated toward high-score (failure-prone) regions, leading to concentration near failures, consistent with Corollary~\ref{cor::detcomXsearch}.

\begin{figure}[b]
\centering

\begin{subfigure}[t]{0.3\linewidth}
  \includegraphics[width=\linewidth]{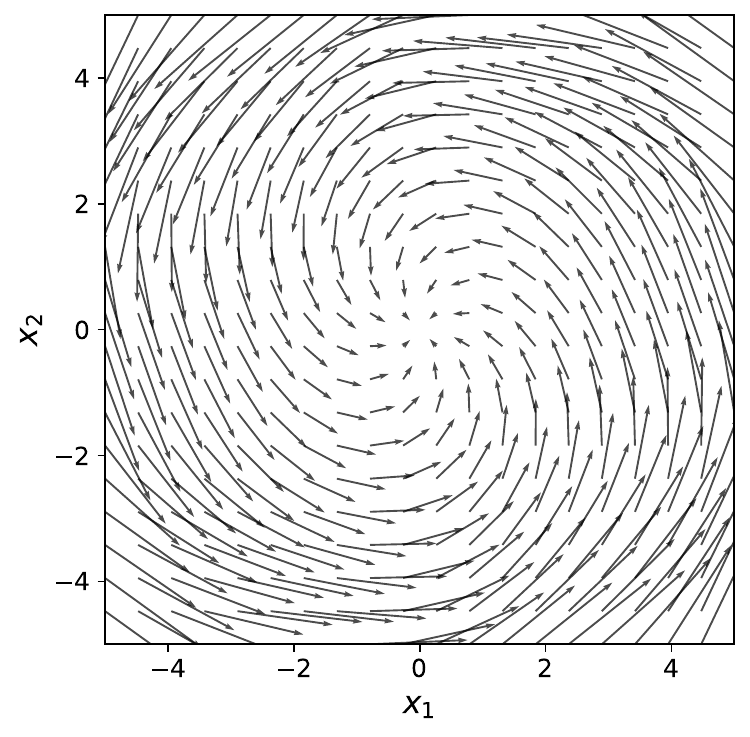}
  \caption{Nonlinear map $x\mapsto y=\Phi(x)$.}
\end{subfigure}
\begin{subfigure}[t]{0.3\linewidth}
  \includegraphics[width=\linewidth]{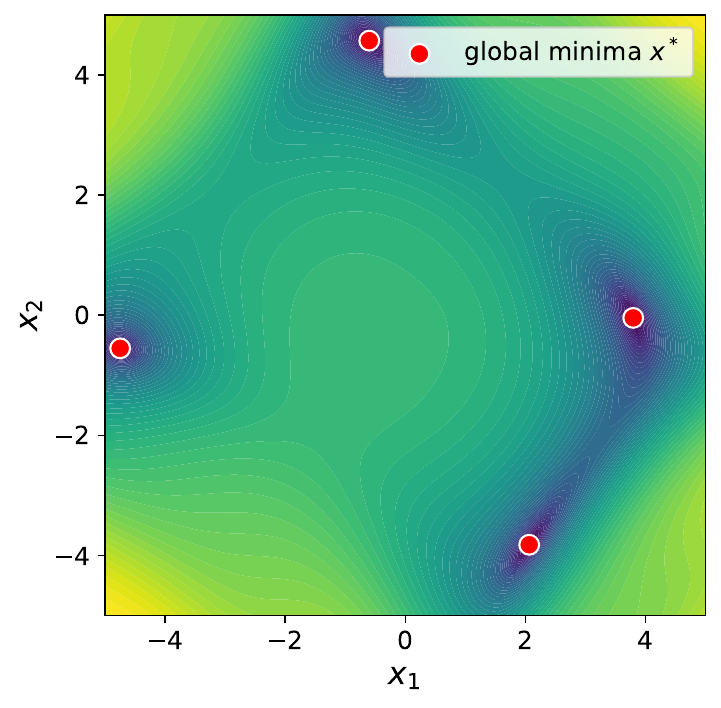}
  \caption{Induced robustness $R(\Phi(x))$ in input space.}
\end{subfigure}

\caption{
Deterministic embedding and induced score landscape. 
Left: nonlinear map $x \mapsto y=\Phi(x)$. 
Right: robustness evaluated on outputs, visualized over the input space as $R(\Phi(x))$. 
}
\label{fig:running_example}
\end{figure}

\begin{figure}[]
\centering

\begin{subfigure}[t]{0.3\linewidth}
  \includegraphics[width=\linewidth]{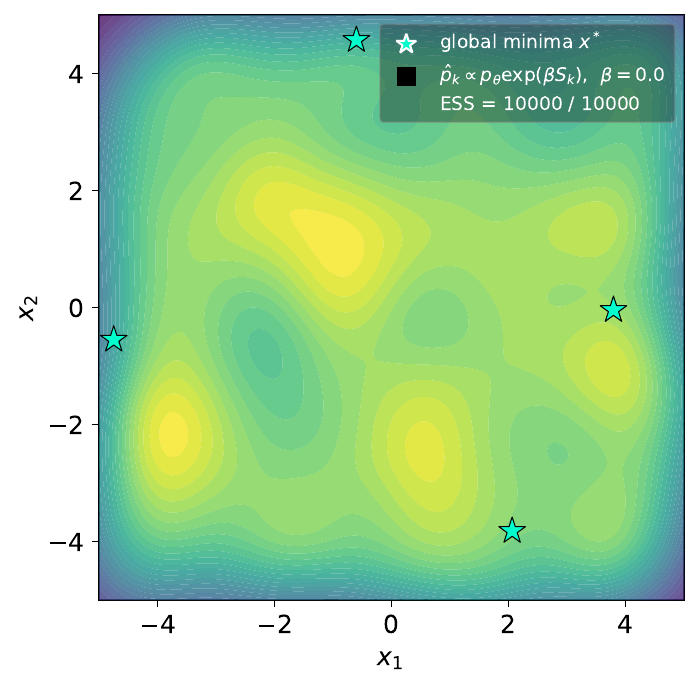}
  \caption{$\beta = 0$}
\end{subfigure}\hfill
\begin{subfigure}[t]{0.3\linewidth}
  \includegraphics[width=\linewidth]{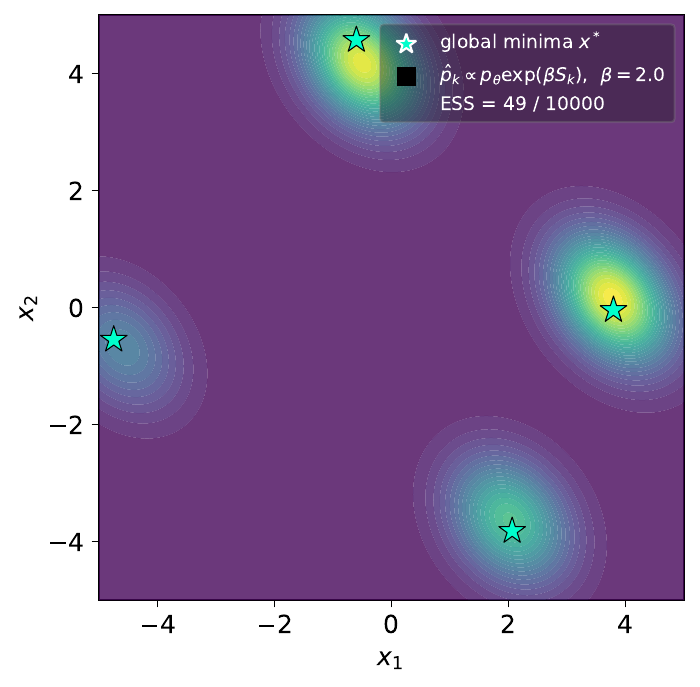}
  \caption{$\beta = 2$}
\end{subfigure}\hfill
\begin{subfigure}[t]{0.3\linewidth}
  \includegraphics[width=\linewidth]{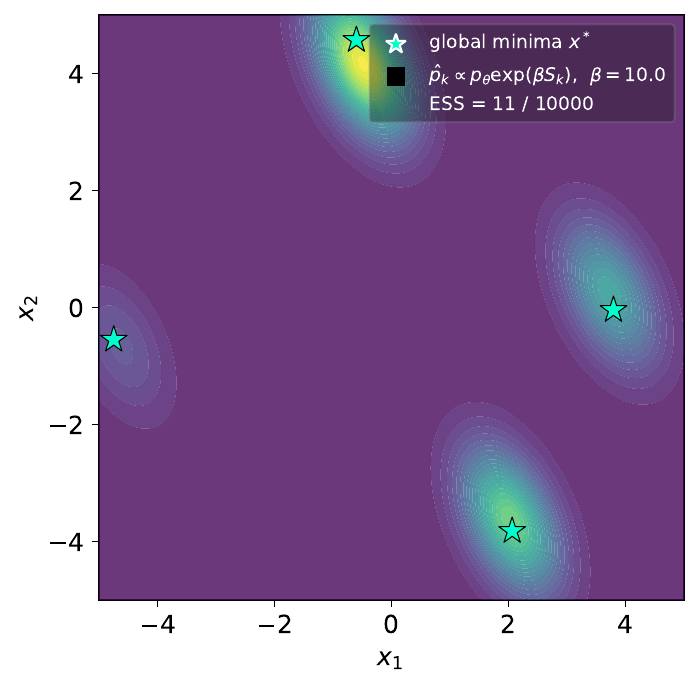}
  \caption{$\beta = 10$}
\end{subfigure}

\vspace{0.6em}

\begin{subfigure}[t]{0.3\linewidth}
  \includegraphics[width=\linewidth]{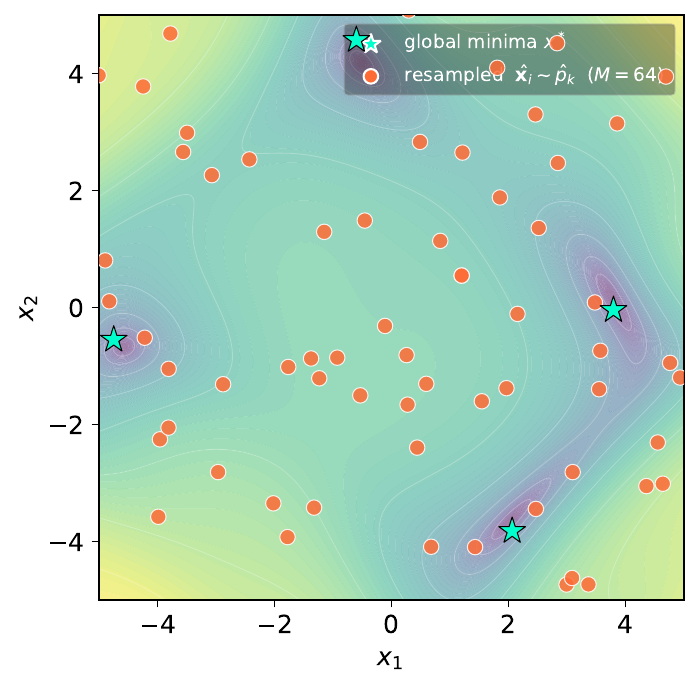}
  \caption{$\beta = 0$}
\end{subfigure}\hfill
\begin{subfigure}[t]{0.3\linewidth}
  \includegraphics[width=\linewidth]{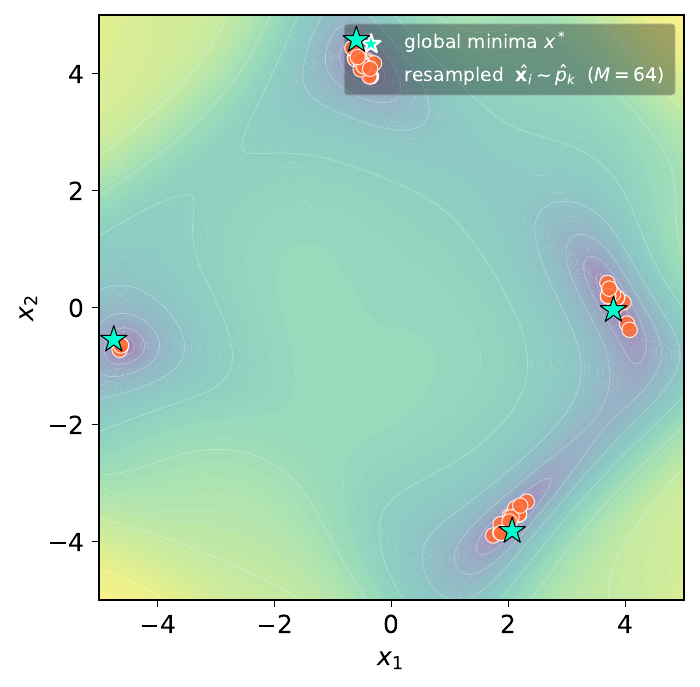}
  \caption{$\beta = 2$}
\end{subfigure}\hfill
\begin{subfigure}[t]{0.3\linewidth}
  \includegraphics[width=\linewidth]{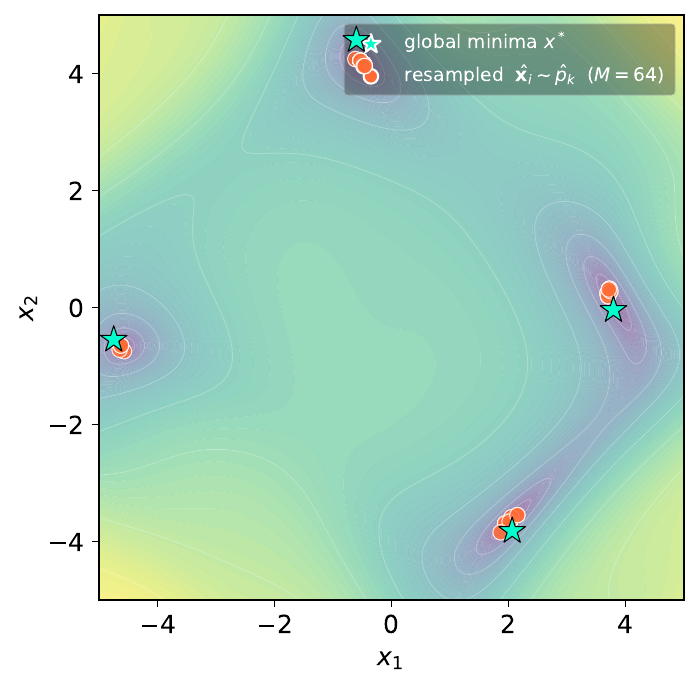}
  \caption{$\beta = 10$}
\end{subfigure}

\caption{
Effect of exponential tilting on the input distribution. 
Top row: reweighted densities proportional to $p_\theta(x)\exp(\beta S(x))$. 
Bottom row: samples obtained by importance sampling from the weighted distribution. 
}
\label{fig:tilting_det}
\end{figure}

\section{Algorithmic Details }\label{sec::implementation}

\linelabel{ln:arch-overview-start}\rtk{Fig.~\ref{fig:architecture} provides an 
overview of the \fw pipeline. \fw operates in two phases. In the offline phase, 
a diffusion model learns the joint prior $p_\theta(x,y)$ over plausible input 
trajectory pairs $(x,y)$, which is then frozen and reused across objectives without retraining. In the 
online phase, at each iteration $k$, candidates are drawn from the tilted density 
$p_k(z) \propto p_\theta(x,y)\exp(\beta_k S_k(x,y))$ via gradient-guided reverse 
diffusion, where $\nabla S_k$ biases each denoising step toward higher-score 
regions of the joint space. Only the selected candidate is then executed through 
the black-box simulator to evaluate the true $R$, with the observation used to 
refine $S_k$ if falsification is not achieved.}
\linelabel{ln:arch-overview-end}
We next describe each stage of the pipeline and then state the complete procedure in Algorithm~\ref{alg:diffusion_falsification}.

\begin{figure*}[]
    \centering
    \includegraphics[width=0.8\textwidth]{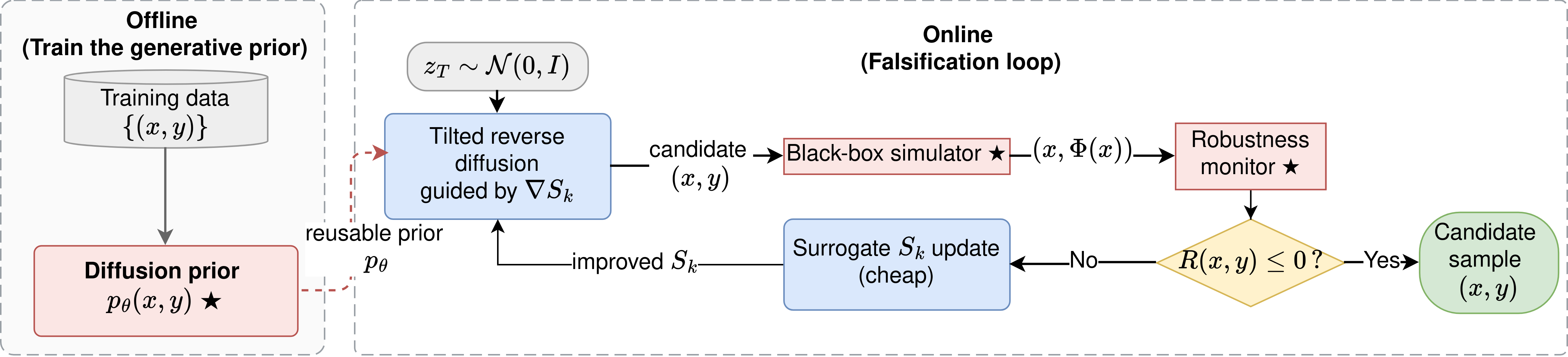}
    \caption{\fw system architecture. \textit{Offline}, a joint diffusion prior $p_\theta(x,y)$ is trained once on randomly sampled $(x,y)$ pairs and frozen. \textit{Online}, tilted reverse diffusion guided by $\nabla S_k$ proposes candidates; the black-box simulator returns $R(x,y)$, which tests the failure condition $R(x,y)\!\le\!0$ and refines the surrogate $S_k$. \textcolor{costred}{Red}~(\,$\star$\,) steps invoke the true system or train the prior (expensive)
    ; \textcolor{costblue}{blue} steps run entirely within the learned model (cheap); grey denotes data; green denotes the falsifying output. 
    }
    \label{fig:architecture}
\end{figure*}

\rtk{\paragraph{Joint Diffusion}
The\linelabel{ln:alg-joint_diff-start} joint prior $p_\theta(z)$ is trained using DDPM~\cite{Ho2020DDPM} on paired samples $(x_0, y_0)$, with $z = (x,y)$, using a coupled architecture adapted from JointNet~\cite{zhang2023jointnet}. Two interacting branches, one for $x$ and one for $y$, each predict their corresponding noise component while conditioning on the other, capturing dependencies between inputs and trajectories. Given paired samples $(x_0, y_0)$, we train with a joint denoising objective
\[
\mathbb{E}\!\left[
\|\epsilon_x - \hat{\epsilon}_x(x_t,y_t,t)\|_2^2
+\|\epsilon_y - \hat{\epsilon}_y(y_t,x_t,t)\|_2^2
\right],
\]
where $(x_t,y_t)$ are obtained via the forward diffusion process. To improve learning stability, the input branch is first pretrained, followed by joint training with cross-conditioning, and finally end-to-end fine-tuning of the full model.
\linelabel{ln:alg-joint_diff-end}
\linelabel{ln:alg-base_sampling-start}
\paragraph{Base Sampling}
Sampling from the learned joint prior $p_\theta(x,y)$ follows the standard DDPM reverse process~\cite{Ho2020DDPM}. Starting from $z_T\sim\mathcal{N}(0,I)$, samples are generated via
\begin{align}\label{eq:ddpm_prev_time_clean}
    z_{t-1} = \mu_\theta(z_t,t) + \sigma_t\eta,
    \quad \eta\sim\mathcal{N}(0,I),
\end{align}
where $\mu_\theta(z_t,t)$ is defined in Eq.~\eqref{eq:ddpm_mean}. This defines the unguided sampler used to initialize the dataset with $B$ samples (Algorithm~\ref{alg:diffusion_falsification}, Lines~\ref{alg:line:init_start}--\ref{alg:line:init_end}).
\linelabel{ln:alg-base_sampling-end}
\linelabel{ln:alg-score-func-start}
\paragraph{Scoring Function}
The surrogate $S_k:\mathcal{X}\times\mathcal{Y}\rightarrow\mathbb{R}$ is a neural network trained on the accumulated dataset $\mathcal{H}$ to approximate $-R(x,y)$, normalized against the empirical standard deviation of robustness values in $\mathcal{H}$ to stabilize gradient magnitudes across iterations. Consistent with the joint formulation in Section~\ref{sec::main:analysis}, the surrogate takes the full pair $(x,y)$ as input, capturing dependencies between the scenario and the induced trace. 
\linelabel{ln:alg-score-func-end}
\linelabel{ln:alg-grad_guided-start}
\paragraph{Gradient-Guided Sampling}
For a differentiable surrogate, the log-gradient of the tilted target decomposes as
\begin{equation}\label{eq:tilted_score}
    \nabla_z\log p_k(z)
    =
    \underbrace{\nabla_z\log p_\theta(z)}_{\text{prior score}}
    +\;
    \beta_k\,\underbrace{\nabla_z S_k(z)}_{\text{surrogate guidance}},
\end{equation}
where the diffusion model provides implicit access to $\nabla_z\log p_\theta(\cdot)$ via Eq.~\eqref{eq:score}, and the surrogate provides $\nabla_z S_k(\cdot)$. Since $S_k$ is defined on clean samples, guidance is applied at each reverse step $t$ by first computing the Tweedie denoised estimate
\[
    \hat{z}_0(z_t,t)
    =
    \frac{1}{\sqrt{\bar{\alpha}_t}}
    \!\left(z_t - \sqrt{1-\bar{\alpha}_t}\,\varepsilon_\theta(z_t,t)\right),
\]
and modifying the reverse mean as 
$\mu^{(k)}(z_t,t)=\mu_\theta(z_t,t) +\gamma_t\beta_k\,\nabla_z S_k(\hat{z}_0(z_t,t))$,
with $\beta_k$ and $S_k$ fixed across all $t$ within iteration $k$. Expressed in the noise-prediction parameterisation:

\begin{align}\label{eq:modified_noise_ddpm_clean}
    \varepsilon^{(k)}(z_t,t)
    =
    \varepsilon_\theta(z_t,t)
    -
    \frac{\sqrt{\alpha_t}\sqrt{1-\bar{\alpha}_t}}{1-\alpha_t}
    \,\gamma_t\beta_k\,
    \nabla_z S_k\!\left(\hat{z}_0(z_t,t)\right).
\end{align}
The guided reverse step is then
\begin{equation}\label{eq:ddpm_guided_step}
    z_{t-1} = \mu^{(k)}(z_t,t) + \sigma_t\eta,
    \quad \eta\sim\mathcal{N}(0,I),
\end{equation}
keeping the diffusion variance unchanged. From $M$ independent noise samples this produces $M$ candidate scenarios.
\linelabel{ln:alg-grad_guided-end}
\linelabel{ln:alg-sel-eval-start}
\paragraph{Selection and Evaluation}
Guided diffusion produces $M$ candidate scenarios reweighted as $p_i \propto \exp(\beta_k S_{k,\theta}(x_i,y_i))$. A candidate is sampled, simulated using $\Phi(x)$, and evaluated (Algorithm~\ref{alg:diffusion_falsification}, Lines~\ref{alg:line:evaluate}--\ref{alg:line:surrogate_retrain}), with the resulting data used to update the surrogate. This process repeats until falsification or budget exhaustion\linelabel{ln:alg-sel-eval-end}}.

\SetKwFor{ForEach}{foreach}{do}{end foreach}
\begin{algorithm}[t]
\footnotesize
\caption{\alg}
\label{alg:diffusion_falsification}
\KwIn{Learned prior \(p_\theta(x,y)\), robustness \(R(x,y)\), batch size \(B\), candidates \(M\), simulation budget, diffusion steps \(T\), guidance weight \(\{\beta_k\}\)}
Initialize dataset \(\mathcal{H}\gets\emptyset\)\;
\For{iteration \(k = 1,2,\dots\)}{
  \If{\(k = 1\)}{
      \For{\(i=1,\dots,B\)\label{alg:line:init_start}}{
          Sample \((x_i,y_i)\sim p_\theta(x,y)\)\;
          Compute robustness \(R_i \gets R(x_i,y_i)\)\;
          Add \(((x_i,y_i),-R_i)\) to dataset \(\mathcal{H}\)\;
          \If{\(R_i < 0\)}{
              \Return \((x_i,y_i)\)\;
          }
      }\label{alg:line:init_end}
  }
  Train surrogate model \(S_{k,\theta}\) on \(\mathcal{H}\)\label{alg:line:surrogate_train}\;

  \If{simulation budget exhausted}{break\;}

  \For{\(i=1,\dots,M\)\label{alg:line:guided_start}}{
      \(z_T \sim \mathcal{N}(0,I)\)\;
      \For{\(t=T,\dots,1\)}{
          Compute guided noise prediction by Eq.~\eqref{eq:modified_noise_ddpm_clean}\;
          Apply the guided DDPM reverse update using Eq.~\eqref{eq:ddpm_prev_time_clean}\;
      }
      Obtain candidate \((x_i,y_i)\) from \(z_0\)\;
  }\label{alg:line:guided_end}

  Define selection distribution \(p_i \propto \exp(\beta_k S_{k,\theta}(x_i,y_i))\)\;
  Sample scenario index \(\hat{i}\sim p_i\) and simulate \(\Phi(x_{\hat{i}})\)\label{alg:line:evaluate}\;
  Compute robustness \(R(x_{\hat{i}},\Phi(x_{\hat{i}}))\)\;
  \If{\(R(x_{\hat{i}},\Phi(x_{\hat{i}}))<0\)}{
      \Return \((x_{\hat{i}},\Phi(x_{\hat{i}}))\)\;
  }
  Add \(\big((x_{\hat{i}},y_{\hat{i}}),-R(x_{\hat{i}},y_{\hat{i}})\big)\) to \(\mathcal{H}\)\;
  Retrain surrogate model \(S_{k+1,\theta}\) on updated \(\mathcal{H}\)\label{alg:line:surrogate_retrain}\;
}
\Return No falsification found within simulation budget\;
\end{algorithm}
\section{Numerical Results}\label{sec::numerical-results}

This section evaluates the proposed \fw framework in the context of falsification. 
We describe the benchmark systems and evaluation protocol, and compare against 
a state-of-the-art falsification method, FReaK~\cite{bak2025fast}, across multiple specifications per system.

\subsection{Benchmark Algorithm and Systems}

We introduce the baseline method and benchmarks, followed by the experimental procedure and scoring definitions.

\paragraph{Baseline}
FReaK~\cite{bak2025fast} is a falsification method based on Koopman operator linearization, which lifts nonlinear dynamics into a higher-dimensional space where they evolve linearly. It constructs a data-driven surrogate model and searches for inputs that maximize the robustness of the negated STL specification using a weighted robustness encoding, reducing the optimization to a linear program. The weights are iteratively refined based on critical points identified through simulations of the true system. FReaK consistently achieves high falsification rates with low simulation budgets across ARCH-COMP benchmarks~\cite{bak2025fast, arch2025khandait}. 
While other approaches exist, no single method uniformly dominates across all problems~\cite{arch2025khandait}. 
Thus, comparing against FReaK provides a meaningful and rigorous evaluation of our approach.

\rtk{\paragraph{Benchmark Systems}
We consider the Automatic Transmission (AT) and Chasing Cars (CC) benchmarks from ARCH-COMP~\cite{arch2025khandait}, along with a tractor-trailer (TT2D) scenario~\cite{kim2026safempd}, additionally introduced here as a falsification benchmark with instances of varying difficulty. We compare against FReaK under a fixed simulation budget and briefly summarize each benchmark next.}

\noindent\textbf{Automatic Transmission (AT)}
The AT system is a deterministic hybrid system modeling an automatic transmission controller with both continuous and discrete dynamics~\cite{arch2025khandait}. The system has two input time-varying signals: throttle $u_t \in [0,100]$[\%] and brake $u_b \in [0,325]$[ft-lb].
Given an input trajectory, the system produces output trajectories consisting of vehicle speed $v$[mph], engine speed $\omega$[RPM], and gear indicators $g_1,\dots,g_4$. The simulation horizon is $50$[s] with a sampling interval of $0.01$[s]. The STL specifications are listed in Table~\ref{tab:stl-specifications}. \linelabel{ln:at-framing}\rtk{AT has multiple STL specifications over fixed system dynamics, allowing us to evaluate the generalization capability of a single learned prior across diverse specifications without retraining. Additionally, specifications like $AT5g$ (Table~\ref{tab:stl-specifications}) also allows us to test the method under discrete dynamics, where gradient-based guidance is predicted to be ineffective due to flat robustness surfaces. }

\noindent\textbf{Chasing Cars (CC)}
The CC is a dynamical system modeling a platoon of five vehicles~\cite{arch2025khandait}. It has two input signals, throttle $u_t \in [0,1]$ and brake $u_b \in [0,1]$, applied to the lead vehicle. 
The remaining vehicles follow a fixed control law based on the state of the preceding vehicle. 
Given an input trajectory, the system produces output trajectories consisting of the longitudinal positions $y_1,\dots,y_5$ of the five vehicles. The simulation horizon is $100$[s] with a sampling interval 
of $0.05$[s]. The STL specifications are listed in Table~\ref{tab:stl-specifications}. \linelabel{ln:cc-framing}\rtk{CC also aids in testing the reusability of a single learned prior across long-horizon and nested specifications without retraining. }

\newcolumntype{L}[1]{>{\raggedright\arraybackslash}p{#1}}
\begin{table*}[t]
\centering
\caption{STL specifications with informal descriptions for the Automatic Transmission (AT), Chasing Cars (CC), and Tractor-Trailer (TT2D) benchmarks. 
$\Box$ denotes ``always'', $\Diamond$ denotes ``eventually'', and $\circ \phi$ denotes $\Diamond_{[0.001,0.1]} \phi$.}
\label{tab:stl-specifications}
\scriptsize
\begin{tabularx}{\textwidth}{L{0.6cm} L{5.5cm} X}
\toprule
\textbf{Key} & \textbf{STL Formula} & \textbf{Informal Description} \\

\midrule

AT1 & $\Box_{[0,20]} v < 120$ 
& Speed remains below 120 mph for the first 20 seconds. \\

AT2 & $\Box_{[0,10]} \omega < 4750$ 
& Engine RPM remains below 4750 for the first 10 seconds. \\




AT5$g$ & $\Box_{[0,30]} \big((\lnot g \land \circ g) \rightarrow \circ \Box_{[0,2.5]} g\big)$ 
& Whenever gear $g \in \{1,2,3,4\}$ is engaged, it remains engaged for at least 2.5 seconds. \\

AT6a & $(\Box_{[0,30]} \omega < 3000) \to (\Box_{[0,4]} v < 35)$ 
& If RPM stays below 3000 over the horizon, then speed stays below 35 mph for 4 seconds. \\

AT6b & $(\Box_{[0,30]} \omega < 3000) \to (\Box_{[0,8]} v < 50)$ 
& If RPM stays below 3000 over the horizon, then speed stays below 50 mph for 8 seconds. \\

AT6c & $(\Box_{[0,30]} \omega < 3000) \to (\Box_{[0,20]} v < 65)$ 
& If RPM stays below 3000 over the horizon, then speed stays below 65 mph for 20 seconds. \\

AT6abc & AT6a $\land$ AT6b $\land$ AT6c 
& All three conditional speed constraints must hold simultaneously. \\

\midrule

CC1 & $\Box_{[0,100]} (y_5 - y_4 \leq 40)$ 
& The distance between car 5 and car 4 is always at most 40 units. \\

CC2 & $\Box_{[0,70]} \Diamond_{[0,30]} (y_5 - y_4 \geq 15)$ 
& At all times up to 70 s, within the next 30 s the distance becomes at least 15. \\

CC3 & $\Box_{[0,80]} ((\Box_{[0,20]} (y_2 - y_1 \leq 20)) \lor (\Diamond_{[0,20]} (y_5 - y_4 \geq 40)))$ 
& At all times, either cars 1–2 stay within 20 units for 20 s, or cars 4–5 become separated by at least 40 within 20 s. \\

CC4 & $\Box_{[0,65]} \Diamond_{[0,30]} \Box_{[0,5]} (y_5 - y_4 \geq 8)$ 
& At all times, within 30 s there exists a 5 s interval where the distance stays above 8. \\

CC5 & $\Box_{[0,72]} \Diamond_{[0,8]} ((\Box_{[0,5]} (y_2 - y_1 \geq 9)) \to (\Box_{[5,20]} (y_5 - y_4 \geq 9)))$ 
& At all times, within 8 s: if cars 1–2 stay at least 9 apart for 5 s, then cars 4–5 must stay at least 9 apart during the following 5–20 s. \\

CCx & $\bigwedge_{i=1}^{4} \Box_{[0,50]} (y_{i+1} - y_i > 7.5)$ 
& All adjacent cars maintain a separation greater than 7.5 units at all times. \\

\midrule

TT2D 
& $\neg \Diamond_{[0,T]} \bigl(\|(x,y)-(x^g,y^g)\|_2 \le \epsilon_p \land |\theta_1-\theta_1^g| \le \epsilon_\theta \land |\theta_2-\theta_2^g| \le \epsilon_\theta \bigr)$ 
& The system never reaches the goal region within the horizon (i.e., position and orientations do not satisfy tolerance bounds simultaneously). \\
\bottomrule
\end{tabularx}
\end{table*}

\noindent\textbf{Tractor-Trailer (TT2D:)}
The TT2D system is a deterministic nonlinear model of a tractor–trailer operating in a bounded 2D environment with obstacles. We adopt the model from~\cite{kim2026safempd} and reformulate it as a falsification problem.
\tk{The system has two control inputs: longitudinal velocity $v \in [-3,3]$[m/s] and steering angle $\delta \in [-0.95,0.95]$[rad]. Given an input trajectory, the system produces state trajectories consisting of position $(x,y)$, tractor heading $\theta_1$, and trailer heading $\theta_2$, within spatial limits $X,Y \in (-32,32)$[m].}

The objective is to generate a feasible trajectory from an initial state $S^s = [x^s, y^s, \theta_1^s, \theta_2^s]$ to a goal state $S^g = [x^g, y^g, \theta_1^g, \theta_2^g]$ within a horizon of $500$\,[s] sampled at time interval $0.25$\,[s]. Feasibility, including collision avoidance and control constraints, is enforced via projection, restricting the search to admissible trajectories. The goal region is defined by position and orientation tolerances:
$\|(x,y)-(x^g,y^g)\|_2 \le \epsilon_p$, 
$|\theta_1-\theta_1^g| \le \epsilon_\theta$, and 
$|\theta_2-\theta_2^g| \le \epsilon_\theta$, 
with angles wrapped to $[-\pi,\pi]$. 
Smaller tolerances correspond to harder instances.

For TT2D, unlike AT and CC, \fw defines a trajectory-level objective
$R(x,y) = - \left(l_T(y_T) + \sum_{t=0}^T l_t(y_t, x_t)\right),$
where $l_T$ penalizes deviation from the goal and $l_t$ accumulates stage costs over the horizon. This highlights a limitation of the current FReaK framework: it cannot directly represent such trajectory-level objectives, as it requires STL-based formulations~\cite{bak2025fast}. Therefore, we approximate it using the negation of goal reachability defined over the goal region (Table~\ref{tab:stl-specifications}). 
\rtk{For consistency, all trajectories generated by \fw are evaluated against the same TT2D STL specification as FReaK, so both methods are compared on identical criteria, with the trajectory-level cost serving only as the guidance signal within \fw.}

We consider six scenarios:
(S1) \textit{no obstacles (easy)} with relaxed tolerances $(\epsilon_p, \epsilon_\theta) = (1.25, 0.5)$,
(S2) \textit{no obstacles (hard)} with tighter tolerances $(0.2, 0.2)$,
(S3) \textit{one obstacle}: a single rectangular obstacle placed between the start and goal regions,
(S4) \textit{wide parking}: a constrained parking configuration with boundaries on the sides and one end,
(S5) \textit{narrow parking}: a parking configuration similar to S4 but with reduced available space,
(S6) \textit{random obstacles}: multiple small obstacles randomly distributed in the environment. \linelabel{ln:tt2d-framing}\rtk{TT2D additionally tests whether \alg can exploit a richer guidance signal than the STL specification alone.}

\paragraph{Experimental Procedure}
For AT and CC, we follow the ARCH-COMP evaluation protocol~\cite{arch2025khandait}, reporting simulations required for falsification within a budget of $B + 30$, with initial batch sizes $B \in \{4,8,16\}$; FReaK results are taken from~\cite{arch2025khandait}. For TT2D, both methods start from identical initial states with a budget of $B + 200$ ($B = 16$). All results are averaged over 10 independent runs.

All\rtk{\linelabel{ln:scoring-func-start} input signals are normalized to $[-1,1]$ before training and sampling, and denormalized to their physical range before simulation. From the generated pair $z_0 = (x, y)$, only $x$ is 
passed to the simulator. For AT and CC, $R(x,y)$ is the STL robustness of the specification under test. For TT2D, $R(x,y)$ is a trajectory-level goal-reaching cost, with falsification confirmed against the TT2D STL specification, so both methods are evaluated on identical criteria. Robustness targets are standardized against the empirical standard deviation of values in $\mathcal{H}$ at each retraining step, stabilizing gradient magnitudes as the dataset grows\linelabel{ln:scoring-func-end}.

The \linelabel{ln:beta-search-start}tilting strength $\beta_k$ was selected via a coarse pilot search over $\{1, 10, 20, 50, 100\}$; values below $1$ had negligible effect and values above $100$ consistently degraded performance. The values used are $\beta = 10$ for TT2D, $\beta = 20$ for AT, and $\beta = 100$ for CC. A systematic sensitivity analysis was not conducted, and adaptive tuning of $\beta$ is identified as a future direction\linelabel{ln:beta-search-end}.}

\subsection{Results and Discussion}\label{subsec::results-and-discussion}
In \rtk{the offline phase, the joint prior is implemented as two coupled UNet branches\footnote{\url{https://huggingface.co/docs/diffusers/api/models/unet}}, one for input signals $x$ and one for output trajectories $y$, connected via zero-initialized skip adapters following the JointNet architecture~\cite{zhang2023jointnet}. The model is trained once for 300 epochs with a DDPM scheduler ($T=1000$), \linelabel{ln:training-cost-start}requiring 
under 4 hours on a single NVIDIA A100 GPU. This offline cost is amortized across all downstream specifications, since the same weights are reused without retraining \linelabel{ln:training-cost-end}.}

\rtk{In the online phase,} we measure performance according to the falsification rate (FR) and mean number of simulations required (S), reported in Tables~\ref{tab:arch_results} and~\ref{tab:tt2d_results}. 
In addition to FR and S, we also report the mean runtime ($\bar{T_f}$) and standard error ($\text{SE}_f$) of falsifying runs for TT2D. 

\rtk{We now present the discussion of the results on the individual benchmarks followed by a summary.}

\begin{table*}[t]
\centering
\scriptsize
\caption{Falsification performance on ARCH benchmarks (AT and CC). 
FR denotes falsification rate (out of 10 runs) and S denotes the mean number of simulations required for successful runs. }
\label{tab:arch_results}
\resizebox{0.7\linewidth}{!}{
\begin{tabular}{@{}crrrrrrrrrrrrrr@{}}
\toprule
\textbf{Benchmark} 
& \multicolumn{2}{c}{\textbf{FReaK}} 
& \multicolumn{6}{c}{\textbf{\alg (500 Points)}} 
& \multicolumn{6}{c}{\textbf{\alg (4000 Points)}} \\

\cmidrule(l){2-3} 
\cmidrule(l){4-9} 
\cmidrule(l){10-15}

\multicolumn{1}{l}{} 
& \multicolumn{2}{l}{} 
& \multicolumn{2}{c}{\textbf{B=4}} 
& \multicolumn{2}{c}{\textbf{B=8}} 
& \multicolumn{2}{c}{\textbf{B=16}} 
& \multicolumn{2}{c}{\textbf{B=4}} 
& \multicolumn{2}{c}{\textbf{B=8}} 
& \multicolumn{2}{c}{\textbf{B=16}} \\
\cmidrule{1-1} 
\cmidrule(l){2-3} 
\cmidrule(l){4-5} 
\cmidrule(l){6-7} 
\cmidrule(l){8-9}
\cmidrule(l){10-11} 
\cmidrule(l){12-13} 
\cmidrule(l){14-15}

\multicolumn{1}{l}{} 
& \multicolumn{1}{r}{\textbf{FR}} 
& \multicolumn{1}{r}{\textbf{S}} 
& \multicolumn{1}{r}{\textbf{FR}} 
& \multicolumn{1}{r}{\textbf{S}} 
& \multicolumn{1}{r}{\textbf{FR}} 
& \multicolumn{1}{r}{\textbf{S}} 
& \multicolumn{1}{r}{\textbf{FR}} 
& \multicolumn{1}{r}{\textbf{S}} 
& \multicolumn{1}{r}{\textbf{FR}} 
& \multicolumn{1}{r}{\textbf{S}} 
& \multicolumn{1}{r}{\textbf{FR}} 
& \multicolumn{1}{r}{\textbf{S}} 
& \multicolumn{1}{r}{\textbf{FR}} 
& \multicolumn{1}{r}{\textbf{S}} \\\midrule
\textbf{AT1} & 10 & 4.8 & 0 & -- & 0 & -- & 0 & -- & 0 & -- & 0 & -- & 0 & -- \\
\textbf{AT2} & 10 & 2.1 & 10 & 6.3 & 10 & 9.2 & 10 & 9.0 & 10 & 8.1 & 10 & 11.2 & 10 & 10.6 \\
\textbf{AT51} & 10 & 8.7 & 0 & -- & 0 & -- & 0 & -- & 0 & -- & 0 & -- & 0 & -- \\
\textbf{AT52} & 10 & 1.3 & 10 & 4.9 & 10 & 6.9 & 10 & 6.7 & 10 & 2.7 & 10 & 3.0 & 10 & 3.5 \\
\textbf{AT53} & 10 & 1.1 & 0 & -- & 0 & -- & 0 & -- & 7 & 16.4 & 8 & 13.0 & 10 & 12.1 \\
\textbf{AT54} & 10 & 2.4 & 0 & -- & 0 & -- & 0 & -- & 0 & -- & 0 & -- & 0 & -- \\
\textbf{AT6a} & 10 & 7.4 & 6 & 21.0 & 6 & 11.17 & 6 & 31.5 & 9 & 19.6 & 9 & 20.9 & 9 & 20.6 \\
\textbf{AT6b} & 10 & 6.2 & 0 & -- & 0 & -- & 2 & 20.0 & 9 & 24.2 & 10 & 24.7 & 10 & 25.5 \\
\textbf{AT6c} & 10 & 5.9 & 4 & 20.8 & 3 & 25.3 & 7 & 20.2 & 10 & 28.3 & 8 & 25.5 & 10 & 23.5 \\
\textbf{AT6abc} & 10 & 6.4 & 5 & 26.4 & 7 & 19.1 & 8 & 27.6 & 10 & 18.2 & 9 & 21.2 & 10 & 22.8\\
\cmidrule(lr){1-1} 
\cmidrule(lr){2-3} 
\cmidrule(lr){4-5} 
\cmidrule(lr){6-7} 
\cmidrule(lr){8-9}
\cmidrule(lr){10-11} 
\cmidrule(lr){12-13} 
\cmidrule(lr){14-15}

\textbf{CC1} & 10 & 3.6 & 6 & 11.2 & 7 & 12.6 & 7 & 17.1 & 9 & 9.6 & 10 & 9.6 & 10 & 12.6 \\
\textbf{CC2} & 10 & 3 & 5 & 10.4 & 6 & 13.9 & 7 & 16.8 & 8 & 8.5 & 9 & 11.9 & 10 & 9.3 \\
\textbf{CC3} & 10 & 5.7 & 6 & 17.9 & 6 & 13.8 & 6 & 24.2 & 9 & 15.6 & 9 & 10.4 & 9 & 18.3 \\
\textbf{CC4} & 10 & 176.7 & 1 & 16.8 & 2 & 16.5 & 2 & 27.4 & 3 & 14.7 & 2 & 15.0 & 3 & 25.3 \\
\textbf{CC5} & 10 & 48.2 & 6 & 10.3 & 6 & 19.6 & 7 & 23.2 & 8 & 8.6 & 8 & 17.3 & 9 & 19.7 \\
\textbf{CCx} & 10 & 110.5 & 5 & 11.6 & 6 & 14.4 & 6 & 22.1 & 10 & 10.3 & 10 & 13.0 & 10 & 18.1 \\
\bottomrule
\end{tabular}
}
\end{table*}

\rtk{\paragraph{Automatic Transmission (AT)} On the AT benchmarks (Table~\ref{tab:arch_results}), \fw is not competitive with FReaK. FReaK falsifies every specification at low simulation counts, helped by a Koopman surrogate that captures the dynamics accurately and AT specifications that need only a few time indices for robustness encoding~\cite{bak2025fast}. \fw fails outright on three specifications, each for an identifiable reason. AT1 follows a broader ARCH-COMP pattern~\cite{arch2025khandait} in which most tools fail under smooth input interpolation. AT51 and AT54 have discrete gear signals that flatten the robustness landscape, since many distinct input signals map to the same robustness value and the guidance gradient becomes uninformative. \linelabel{ln:at-flatgrad}\rgp{This is exactly the regime our analysis flags as unfavorable, since Theorem~\ref{thm::faiAmplv2} guarantees that tilting raises the failure probability only when the surrogate tends to score failing scenarios above non-failing ones, an ordering that a flattened landscape erases.} This effect, however, is specification-specific rather than systematic, as the remaining gear specifications AT52 and AT53 are still falsified at 4000 points. We did not run experiments targeting these flat-gradient cases and note them as future work\linelabel{ln:at-flatgrad-end}.}

\paragraph{Chasing Cars (CC)} On CC1-CC3, shown in Table~\ref{tab:arch_results}, both methods achieve high falsification rates at 4000 points, though FReaK requires fewer simulations. The picture changes on CC4, CC5, and CCx, which involve longer horizons and nested specifications. 
These require a large number of time indices for accurate robustness encoding~\cite{bak2025fast}, making FReaK's optimization increasingly expensive and forcing a coarser discretization that reduces surrogate accuracy. \linelabel{ln:CC-theory-conn}\rgp{\alg avoids this cost, since it is guided by the surrogate score of the STL robustness evaluated on the realized trace rather than by a logical encoding of the formula, so the nesting and horizon length that drive up the baseline's optimization do not enter \alg's per-iteration cost}\linelabel{ln:CC-theory-conn-end}. On CCx, \alg achieves full falsification in 10-18 simulations compared to 110.5 for FReaK. On CC5, \alg reaches comparable falsification rates with far fewer simulations. On CC4, \alg does not match FReaK's falsification rate but requires far fewer simulations in successful runs.

\paragraph{Tractor-Trailer (TT2D)} 

On TT2D, \alg at 4000 points achieves falsification success rate of 1.0 across all six scenarios. In S1, S2, S3, and S6, FReaK succeeds in only a fraction of runs requiring significantly more simulations and runtime, often by an order of magnitude. In contrast, \alg maintains success rate of 1.0 on all four scenarios with far fewer simulations and substantially lower runtime. 
\linelabel{ln:TT2D-theory-conn}\rgp{This holds for two reasons. The first is that the score-based robustness is better able to guide search in environments where the admissible search space is large but the falsifying region constitutes only a small subset, a case where optimization-based approaches tend to struggle. A second reason is the guidance signal itself. Our analysis requires only that the score rank failing scenarios above non-failing ones, so \alg can be driven by the richer trajectory-level goal-reaching cost while still being confirmed against the same goal-reachability specification as FReaK, and these results show that this richer signal succeeds where an objective tied to the specification alone stalls.}\linelabel{ln:TT2D-theory-conn-end}
\rgp{The advantage does not extend to Wide (S4) and Narrow Parking (S5), where model-enforced feasibility projections tightly restrict admissible trajectories and both methods perform comparably. This suggests difficulty in these scenarios is driven more by the geometry of the feasible set than by the search procedure.} On these two scenarios, FReaK exhibits lower runtime than \fw, highlighting its advantage in such scenarios.

\begin{table*}
\centering
\caption{Falsification performance on TT2D scenarios. 
FR denotes falsification rate (out of 10 runs), S the mean number of simulations for successful runs, $\bar{T_f}$ the mean runtime (seconds), and $\text{SE}_f$ its standard error over falsifying runs. }
\label{tab:tt2d_results}
\resizebox{0.9\textwidth}{!}{%
\begin{tabular}{lrrrrrrrrrrrr}
\toprule
\textbf{Benchmark}
& \multicolumn{4}{c}{\textbf{FReaK}}
& \multicolumn{4}{c}{\textbf{\alg (500 Points, B=16)}}
& \multicolumn{4}{c}{\textbf{\alg (4000 Points, B=16)}} \\
\cmidrule(l){2-5}
\cmidrule(l){6-9}
\cmidrule(l){10-13}
& \textbf{FR} & \textbf{S}
& \textbf{$\bar{T_f}$} & \textbf{$\text{SE}_f$}
& \textbf{FR} & \textbf{S}
& \textbf{$\bar{T_f}$} & \textbf{$\text{SE}_f$}
& \textbf{FR} & \textbf{S}
& \textbf{$\bar{T_f}$} & \textbf{$\text{SE}_f$} \\
\midrule
\textbf{(S1) No Obs (Easy)}
& 8 & 98.125 & 3131.35 & 627.1
& 10 & 46.7 & 1641.26 & 534.46
& 10 & 19.8 & 238.15 & 55.9 \\
\textbf{(S2) No Obs (Hard)}
& 1 & 133 & 3982.7 & -
& 0 & - & - & -
& 10 & 30.9 & 1562.87 & 738.64 \\
\textbf{(S3) Single Obs}
& 3 & 120.67 & 4280.3 & 1072.42
& 5 & 69.8 & 2345.69 & 964.52
& 10 & 37.7 & 1174.02 & 685.76 \\
\textbf{(S4) Wide Parking}
& 10 & 26.8 & 1205.88 & 350.93
& 2 & 110 & 6993.06 & 6912.36
& 10 & 56.3 & 2146.61 & 512.92 \\
\textbf{(S5) Narrow Parking}
& 10 & 50.8 & 2141.02 & 827.19
& 1 & 129 & 3805.55 & -
& 10 & 46.6 & 2507.91 & 800.44 \\
\textbf{(S6) Random Obs}
& 1 & 110 & 5382.2 & -
& 0 & - & - & -
& 10 & 44.9 & 1856.38 & 790.54 \\
\bottomrule
\end{tabular}
}
\end{table*}

The 500-point results reflect the outcome of insufficient trajectory coverage during training. When training data does not cover the trajectory space well, the learned prior cannot provide reliable guidance, and falsification rates drop. This represents an upfront investment, but one that is amortized as the number of specifications increase since the prior is reusable across tasks without retraining. Whereas FReaK requires optimization from scratch per specification. Notably, even with 500 points, \alg exceeds FReaK on S1 and S3, despite using a naive implementation of the JointNet~\cite{zhang2023jointnet}. 

\paragraph{Summary} Overall, \alg provides two key advantages for the benchmarks considered. First, instead of formulating and solving a new optimization problem for each specification, it learns a joint prior over inputs and trajectories that captures system behavior and can be reused across specifications. Second, falsification is performed by guiding this prior using any function that evaluates the resulting execution (e.g., a robustness measure, reward, objective function, or expert-defined score) without requiring STL encodings or decomposition into logical primitives, making the method applicable across a wide range of tasks.

\section{Conclusion}\label{sec::conclusion}


\fw casts \rtk{falsification as the exponential tilting of a diffusion-induced joint distribution over environments and executions toward failures. This tilting is the KL-optimal reweighting of the prior, and it admits an exact importance-sampling interpretation in the joint space. It provably amplifies failure probability under mild score-ranking assumptions, and strictly overcomes the multiplicative rarity bottleneck that caps conditional sampling. For deterministic systems, it reduces to optimal importance sampling over inputs along the dynamics manifold, so every proposed counterexample is validated by the true simulator while the prior need never match the system under test.

Empirically, a single joint prior trained once and reused without retraining remains competitive with state-of-the-art falsification on the standard ARCH-COMP specifications and outperforms it on long-horizon, nested, and non-STL objectives. The simulator is called only on selected candidates, and those evaluations train the surrogate whose gradients guide denoising. Expensive system calls therefore stay confined to scoring, keeping the procedure sample-efficient and compatible with black-box pipelines.


Several aspects require further investigation. The tilting strength $\beta$ is presently fixed per benchmark from a coarse pilot search, leaving a principled adaptive schedule as a clear next step. Large $\beta$ drives mass into high-score regions where the prior is least accurate and can push denoised samples off the data manifold. \linelabel{ln:mismatch-start}Because only $x$ is simulated, the gap between the generated trace and the true rollout $\Phi(x)$ is observable at no extra cost, and feeding it back as a mismatch-correction term during denoising would keep samples on-manifold and refine the prior online. \linelabel{ln:ess-future-start}Similarly, guided diffusion shifts the proposal rather than reweighting a fixed sample set, so ESS degradation does not affect our results, though a formal account of this interaction remains open\linelabel{ln:ess-future-end}.

The surrogate is a second lever. \linelabel{ln:conclusion-hybrid-lim-start}In discrete or hybrid dynamics, such as the AT5 gear signals, the surrogate gradient is uninformative and sampling reverts to the unguided prior, preserving coverage but not driving failures\linelabel{ln:conclusion-hybrid-lim-end}, \linelabel{ln:conclusion-hybrid-fut-start}motivating surrogates that exploit hybrid structure and samplers that retain diversity where gradients vanish\linelabel{ln:conclusion-hybrid-fut-end}. 
On the systems side, DDPM inference and the coupled architecture limit scaling, motivating faster samplers such as DDIM and more efficient joint architectures. Finally, the extension of the analysis to stochastic systems, alongside applications to planning, synthesis, and robotics, are natural directions for 
\fw}.

\bibliographystyle{splncs04}
\bibliography{emsoft_ref}

@book{DemboZeitouni1998LDP,
  title     = {Large Deviations Techniques and Applications},
  author    = {Dembo, Amir and Zeitouni, Ofer},
  year      = {1998},
  publisher = {Springer},
  doi       = {10.1007/978-1-4612-5320-4}
}

@inproceedings{Ho2020DDPM,
  title     = {Denoising Diffusion Probabilistic Models},
  author    = {Ho, Jonathan and Jain, Ajay and Abbeel, Pieter},
  booktitle = {Advances in Neural Information Processing Systems},
  year      = {2020},
  doi       = {10.48550/arXiv.2006.11239}
}

@inproceedings{dreossi2018semantic,
  title={Semantic adversarial deep learning},
  author={Dreossi, Tommaso and Jha, Somesh and Seshia, Sanjit A},
  booktitle={International Conference on Computer Aided Verification},
  pages={3--26},
  year={2018},
  organization={Springer}
}

@inproceedings{vin20233d,
  title={3d environment modeling for falsification and beyond with scenic 3.0},
  author={Vin, Eric and Kashiwa, Shun and Rhea, Matthew and Fremont, Daniel J and Kim, Edward and Dreossi, Tommaso and Ghosh, Shromona and Yue, Xiangyu and Sangiovanni-Vincentelli, Alberto L and Seshia, Sanjit A},
  booktitle={International Conference on Computer Aided Verification},
  pages={253--265},
  year={2023},
  organization={Springer}
}

@inproceedings{fremont2020formal,
  title={Formal analysis and redesign of a neural network-based aircraft taxiing system with VerifAI},
  author={Fremont, Daniel J and Chiu, Johnathan and Margineantu, Dragos D and Osipychev, Denis and Seshia, Sanjit A},
  booktitle={International Conference on Computer Aided Verification},
  pages={122--134},
  year={2020},
  organization={Springer}
}

@book{bucklew2004introduction,
  title={Introduction to rare event simulation},
  author={Bucklew, James Antonio and Bucklew, J},
  volume={5},
  year={2004},
  publisher={Springer}
}

@article{csiszar1975idivergence,
  title={I-divergence geometry of probability distributions and minimization problems},
  author={Csisz{\'a}r, Imre},
  journal={The Annals of Probability},
  volume={3},
  number={1},
  pages={146--158},
  year={1975},
  doi={10.1214/aop/1176996454}
}

@book{cover2006elements,
  title={Elements of Information Theory},
  author={Cover, Thomas M. and Thomas, Joy A.},
  year={2006},
  publisher={Wiley},
  doi={10.1002/047174882X}
}

@article{donsker1975asymptotic,
  title={Asymptotic evaluation of certain Markov process expectations},
  author={Donsker, Monroe D. and Varadhan, S. R. Srinivasa},
  journal={Communications on Pure and Applied Mathematics},
  volume={28},
  number={1},
  pages={1--47},
  year={1975},
  doi={10.1002/cpa.3160280102}
}

@inproceedings{song2019generative,
  title={Generative Modeling by Estimating Gradients of the Data Distribution},
  author={Song, Yang and Ermon, Stefano},
  booktitle={Advances in Neural Information Processing Systems},
  year={2019},
  doi={10.48550/arXiv.1907.05600}
}

@inproceedings{song2021scorebased,
  title={Score-Based Generative Modeling through Stochastic Differential Equations},
  author={Song, Yang and Sohl-Dickstein, Jascha and Kingma, Diederik P. and Kumar, Abhishek and Ermon, Stefano and Poole, Ben},
  booktitle={International Conference on Learning Representations},
  year={2021},
  doi={10.48550/arXiv.2011.13456}
}

@inproceedings{dhariwal2021diffusion,
  title={Diffusion Models Beat GANs on Image Synthesis},
  author={Dhariwal, Prafulla and Nichol, Alex},
  booktitle={Advances in Neural Information Processing Systems},
  year={2021},
  doi={10.48550/arXiv.2105.05233}
}

@inproceedings{zhang2023jointnet,
title = {JointNet: Extending Text-to-Image Diffusion for Dense Distribution Modeling},
booktitle = {ICLR},
author = {Jingyang Zhang and Shiwei Li and Yuanxun Lu and Tian Fang and David McKinnon and Yanghai Tsin and Long Quan and Yao Yao},
year = {2024},
URL = {https://arxiv.org/abs/2310.06347}
}

@inproceedings{arch2025khandait,
  author    = {Tanmay Khandait and Deyun Lyu and Paolo Arcaini and Georgios Fainekos and Federico Formica and Sauvik Gon and Abdelrahman Hekal and Atanu Kundu and Claudio Menghi and Giulia Pedrielli and Rajarshi Ray and Quinn Thibeault and Masaki Waga and Zhenya Zhang},
  title     = {ARCH-COMP25 Category Report: Falsification},
  booktitle = {Proceedings of 12th Int. Workshop on Applied Verification for Continuous and Hybrid Systems},
  editor    = {Goran Frehse and Matthias Althoff},
  series    = {EPiC Series in Computing},
  volume    = {108},
  publisher = {EasyChair},
  bibsource = {EasyChair, https://easychair.org},
  issn      = {2398-7340},
  url       = {/publications/paper/xX5W},
  doi       = {10.29007/dgnn},
  pages     = {169-189},
  year      = {2025}}

@InProceedings{bak2025fast,
author="Bak, Stanley
and Hekal, Abdelrahman
and Kochdumper, Niklas
and Lew, Ethan
and Mata, Andrew
and Rahmati, Amir",
editor="Akshay, S.
and Niemetz, Aina
and Sankaranarayanan, Sriram",
title="Fast Koopman Surrogate Falsification Using Linear Relaxations and Weights",
booktitle="Automated Technology for Verification and Analysis",
year="2025",
publisher="Springer Nature Switzerland",
address="Cham",
pages="234--255",
isbn="978-3-031-78750-8"
}

@inproceedings{kim2026safempd, 
	  author    = {Kim, Taekyung and Majd, Keyvan and Okamoto, Hideki and Hoxha, Bardh and Panagou, Dimitra and Fainekos, Georgios},
	  title     = {Safe Model Predictive Diffusion with Shielding},
    booktitle = {IEEE International Conference on Robotics and Automation (ICRA)},
    shorttitle = {Safe MPD},
    year      = {2026}
}

\section{Appendix}\label{sec::appendix}
In this appendix we provide the detailed proofs of the results not shown in the main manuscript.
\vspace{-2pt}
\subsection{Proof of Proposition~\ref{prop::condBottleneck}}\label{app:prop_1_proof}
\begin{proof}
Let $(X,Y)$ denote one draw from the conditional sampling procedure:
$X\sim q,$ and $Y\sim p_\theta(\cdot\mid X).$
By assumption $p_\theta(y\mid x)=p^\star(y\mid x)$ for all $x$, so for any measurable $B\subset\mathcal Y$,
\[
\mathbb P(Y\in B\mid X=x)=\int_B p_\theta(y\mid x)\,dy=\int_B p^\star(y\mid x)\,dy.
\]
Define the conditional failure probability under the true dynamics as
\[
f^\star(x)\ :=\ \int_{\mathcal Y}\mathbf 1_F(x,y)\,p^\star(y\mid x)\,dy
\ =\ \mathbb P\big((x,Y)\in F \mid X=x\big).
\]
By assumptions (ii)--(iii), $f^\star(x)=\varepsilon$ for $x\in A$ and $f^\star(x)=0$ for $x\notin A$.

First, we show how, under $p^\star$ we have
\begin{align*}
p^\star(F)
=\int_{\mathcal X} p^\star(x)\,f^\star(x)\,dx
=\int_A p^\star(x)\,\varepsilon\,dx
&=\varepsilon\,p^\star(X\in A)
\end{align*}
which proves $p^\star(F)=\delta\varepsilon$ since $p^\star(X\in A)=\delta$.

We now derive the failure probability induced by the conditional sampler. Under the sampler distribution induced by $q$ and $p_\theta(\cdot\mid x)=p^\star(\cdot\mid x)$,
\[
\mathbb P(F)
=\int_{\mathcal X} q(x)\,f^\star(x)\,dx
=\int_A q(x)\,\varepsilon\,dx
=\varepsilon\,q(A).
\]
If $q(A)\le c\,\delta$, then
$\mathbb P(F)\le c\,\delta\varepsilon.$

Finally, let $T$ be the number of i.i.d. draws until the first failure event occurs. 
Each draw can be modeled as an independent Bernoulli trial with success probability $\mathbb P(F)$, we have $T\sim\mathrm{Geom}(\mathbb P(F))$ (counting trials until first success), and therefore:

$\mathbb E[T]=\frac{1}{\mathbb P(F)}\ \ge\ \frac{1}{c\,\delta\varepsilon}
=\Omega\big((\delta\varepsilon)^{-1}\big).$
\end{proof}

\subsection{ Proof of Lemma~\ref{lem::admisstilt}}\label{app:lemma_1_proof}
\begin{proof}
Let $Z(\beta):=\int_{\mathcal Z} p(z)\exp(\beta S(z))\,dz$ and define
$\mathcal B:=\{\beta\in\mathbb R:\ Z(\beta)<\infty\}$.

\paragraph{Non-emptiness.}
Since $p$ is a probability density, $\int p(z)\,dz=1$. Hence
\begin{align}
Z(0)=\int p(z)\exp(0)\,dz=\int p(z)\,dz=1<\infty,
\end{align}
so $0\in\mathcal B$.

\paragraph{Interval property.}
We show that $\mathcal B$ is an interval (i.e., convex). Let $\beta_1,\beta_2\in\mathcal B$ and
$\alpha\in[0,1]$. By Hölder's inequality,
\begin{align*}
&{}Z(\alpha\beta_1 + (1-\alpha)\beta_2) = ...\\
&=\int p(z)\exp\!\big((\alpha\beta_1+(1-\alpha)\beta_2)S(z)\big)\,dz\\
&= \int p(z)\exp(\alpha\beta_1 S(z))\exp((1-\alpha)\beta_2 S(z))\,dz\\
&\le \Big(\int p(z)\exp(\beta_1 S(z))\,dz\Big)^\alpha
     \Big(\int p(z)\exp(\beta_2 S(z))\,dz\Big)^{1-\alpha}\\
&= Z(\beta_1)^\alpha Z(\beta_2)^{1-\alpha} <\infty.
\end{align*}
so $\alpha\beta_1+(1-\alpha)\beta_2\in\mathcal B$. Thus $\mathcal B$ is convex, hence an interval
(possibly unbounded). Because $0\in\mathcal B$, there exist extended reals
$-\infty\le \beta_- \le 0 \le \beta_+ \le +\infty$ such that
$\mathcal B=(\beta_-,\beta_+)$ 
or possibly with closed endpoints if finiteness holds at an endpoint. 
For simplicity we denote $\mathcal B=(\beta_-,\beta_+)$ as in the statement (the open-interval form
always holds for $\mathrm{int}(\mathcal B)$).

\paragraph{Validity of the tilted density.}
Fix any $\beta\in\mathcal B$. Since $\exp(\beta S(z))\ge 0$ and $p(z)\ge 0$,
we have $p_\beta(z)\ge 0$. Moreover
\[
\int_{\mathcal Z} p_\beta(z)\,dz
=\frac{1}{Z(\beta)}\int_{\mathcal Z} p(z)\exp(\beta S(z))\,dz
=\frac{Z(\beta)}{Z(\beta)}=1,
\]
so $p_\beta$ is a valid probability density on $\mathcal Z$.

\paragraph{Differentiability of $\psi(\beta)=\log Z(\beta)$ on $\mathrm{int}(\mathcal B)$.}
Fix $\beta\in\mathrm{int}(\mathcal B)$ and assume there exists $\delta>0$ such that
$Z(\beta+\delta)<\infty$. For $h\in(0,\delta)$ define
\[
f_h(z):=\frac{\exp((\beta+h)S(z))-\exp(\beta S(z))}{h}.
\]
For each fixed $z$, $f_h(z)\to S(z)\exp(\beta S(z))$ as $h\to 0^+$ by the mean value theorem.
Moreover, for $0<h<\delta$, the mean value theorem yields
\begin{align*}
|f_h(z)| &= |S(z)|\exp((\beta+\xi)S(z))\\
&\le |S(z)|\big(\exp((\beta+\delta)S(z))+\exp(\beta S(z))\big)
\end{align*}
for some $\xi\in(0,h)$. Under the stated assumption $Z(\beta+\delta)<\infty$ and $\beta\in\mathcal B$,
the function $p(z)\exp((\beta+\delta)S(z))$ is integrable, and the above bound provides an integrable
dominating function whenever the right-hand side is integrable (standard sufficient condition; see below).
Hence, by differentiation under the integral sign justified via dominated convergence,
\[
Z'(\beta)=\int p(z)S(z)\exp(\beta S(z))\,dz.
\]
Since $Z(\beta)\in(0,\infty)$ for $\beta\in\mathcal B$, $\psi(\beta)=\log Z(\beta)$ is finite and
\begin{align*}
\psi'(\beta)=\frac{Z'(\beta)}{Z(\beta)}
&=\frac{\int p(z)S(z)\exp(\beta S(z))\,dz}{\int p(z)\exp(\beta S(z))\,dz}\\
&=\int S(z)\,p_\beta(z)\,dz
=\mathbb E_{p_\beta}[S(Z)].
\end{align*}

\paragraph{Second derivative and variance identity.}
Assume additionally that for some $\delta>0$,
$\int p(z)\exp((\beta+\delta)|S(z)|)\,dz<\infty$.
This exponential integrability implies $\int p(z)S(z)^2\exp(\beta S(z))\,dz<\infty$,
allowing a second differentiation under the integral sign:
\[
Z''(\beta)=\int p(z)S(z)^2\exp(\beta S(z))\,dz.
\]
Therefore,
\begin{align*}
\psi''(\beta)
&=\frac{Z''(\beta)}{Z(\beta)}-\Big(\frac{Z'(\beta)}{Z(\beta)}\Big)^2\\
&=\mathbb E_{p_\beta}[S(Z)^2]-\big(\mathbb E_{p_\beta}[S(Z)]\big)^2=\mathrm{Var}_{p_\beta}(S(Z))\ \ge\ 0,
\end{align*}
which completes the proof.
\end{proof}

\subsection{Proof of Proposition~\ref{prop::dynjoint}}\label{app:prop_4_proof}
\begin{proof}
    Assume the base joint distribution $p$ is supported on the manifold
    $\mathcal M:=\{(x,y)\in\mathcal X\times\mathcal Y:\ y=\Phi(x)\}$.
    Equivalently, the conditional law of $Y$ given $X=x$ is degenerate at $\Phi(x)$, i.e.,
    $p(dy\mid x)=\delta_{\Phi(x)}(dy),$
    and the joint factorizes as
    $p(dx,dy)=p_X(dx)\,\delta_{\Phi(x)}(dy).$
    In density notation (with respect to $dx\,dy$) this is written as
    $p(x,y)=p_X(x)\,\delta(y-\Phi(x))$.
    
    Fix $\beta\in\mathcal B$ so the tilt is well-defined and define the tilted joint
    \begin{align*}
    p_\beta(dx,dy) &= \frac{e^{\beta S(x,y)}}{Z(\beta)}\,p(dx,dy),\\
    Z(\beta) &= \iint e^{\beta S(x,y)}\,p(dx,dy).
    \end{align*}
    To compute the induced marginal over inputs, integrate out $y$:
    \begin{align*}
    p_{\beta,X}(dx)
    &:=\int_{\mathcal Y} p_\beta(dx,dy)\\
    &=\frac{1}{Z(\beta)}\int_{\mathcal Y} e^{\beta S(x,y)}\,p_X(dx)\,\delta_{\Phi(x)}(dy)\\
    &=\frac{1}{Z(\beta)}\,p_X(dx)\,e^{\beta S(x,\Phi(x))}.
    \end{align*}
    Hence, in density form,
    \[
    p_{\beta,X}(x)\ \propto\ p_X(x)\exp\!\big(\beta S(x,\Phi(x))\big),
    \]
    which is exactly input-space exponential tilting with the score restricted to the dynamics manifold.
\end{proof}

\end{document}